\documentclass[11pt]{article}

\usepackage{microtype}
\usepackage{graphicx}
\usepackage{subcaption}
\usepackage{booktabs} 
\usepackage{amssymb,amsmath,amsthm,sectsty,url}
\usepackage[letterpaper,hmargin=1.0in,vmargin=1.0in]{geometry}
\usepackage[pdftex,colorlinks,linkcolor=blue,citecolor=blue,filecolor=blue,urlcolor=blue]{hyperref}
\usepackage{cleveref}
\usepackage{color}
\usepackage[boxed]{algorithm}
\usepackage[labelfont=bf]{caption}
\usepackage{tikz}
\usepackage{graphicx}
\usepackage{nicefrac}
\usepackage{aliascnt}

\usepackage{float}
\usepackage{pgfplots}
\usepackage{pgfplotstable}
\pgfplotsset{compat=1.18}
\usepgfplotslibrary{fillbetween}

\newtheorem{theorem}{Theorem}[section]
\crefname{theorem}{Theorem}{Theorems}

\newaliascnt{lemma}{theorem}
\newtheorem{lemma}[lemma]{Lemma}
\aliascntresetthe{lemma}
\crefname{lemma}{Lemma}{Lemmas}

\newaliascnt{proposition}{theorem}
\newtheorem{proposition}[proposition]{Proposition}
\aliascntresetthe{proposition}
\crefname{proposition}{Proposition}{Propositions}

\newaliascnt{corollary}{theorem}

\aliascntresetthe{corollary}
\crefname{corollary}{Corollary}{Corollaries}

\newaliascnt{fact}{theorem}
\newtheorem{fact}[fact]{Fact}
\aliascntresetthe{fact}
\crefname{fact}{Fact}{Facts}

\newaliascnt{definition}{theorem}
\newtheorem{definition}[definition]{Definition}
\aliascntresetthe{definition}
\crefname{definition}{Definition}{Definitions}

\newaliascnt{remark}{theorem}

\aliascntresetthe{remark}
\crefname{remark}{Remark}{Remarks}

\newaliascnt{conjecture}{theorem}

\aliascntresetthe{conjecture}
\crefname{conjecture}{Conjecture}{Conjectures}

\newaliascnt{claim}{theorem}

\aliascntresetthe{claim}
\crefname{claim}{Claim}{Claims}

\newaliascnt{question}{theorem}

\aliascntresetthe{question}
\crefname{question}{Question}{Questions}

\newaliascnt{exercise}{theorem}

\aliascntresetthe{exercise}
\crefname{exercise}{Exercise}{Exercises}

\newaliascnt{example}{theorem}

\aliascntresetthe{example}
\crefname{example}{Example}{Examples}

\newaliascnt{notation}{theorem}

\aliascntresetthe{notation}
\crefname{notation}{Notation}{Notations}

\newaliascnt{problem}{theorem}

\aliascntresetthe{problem}
\crefname{problem}{Problem}{Problems}

\newcommand{\norm}[1]{\lVert#1\rVert}

\def\E{\mathbb E}

\newcommand{\N}{\mathbb N}
\newcommand{\R}{\mathbb R}
\newcommand{\Z}{\mathbb Z}

\usepackage{stix}
\newcommand{\alibi}{L^{\star}}

\newcommand{\kk}{\mathrm{ker}}

\newcommand*\samethanks[1][\value{footnote}]{\footnotemark[#1]}

\title{Positional LSH: \\
Binary Block Matrix Approximation for Attention with Linear Biases}

\author{
  Daniel Wolfson\thanks{Blavatnik School of Computer Science and AI, Tel Aviv University.
  \texttt{\{wolfson2@mail,talwag@tauex\}.tau.ac.il}
  }
  \and 
  Tal Wagner\samethanks
}

\date{} 

\begin{document}
\maketitle

\maketitle

\begin{abstract}
Positional encoding in transformers is commonly implemented through positional embeddings, attention masks, or bias terms, but formal connections between these mechanisms remain limited. We study attention with positional bias through the lens of locality-sensitive hashing (LSH), focusing on Attention with Linear Biases (ALiBi). We show that the ALiBi bias matrix is the expectation of contiguous block-diagonal binary masks induced by a ``positional LSH'' scheme. The empirical mean of masks sampled from this scheme yields spectral norm and max-norm approximation guarantees with bounded block sizes with high probability. This structural theorem implies an efficient approximation algorithm for ALiBi-biased attention: with high probability over the sampled masks, the approximate attention output is accurate simultaneously for all query-key-value inputs and can be computed in near-linear time in the context length, reducing long-context ALiBi to a collection of randomized short-context regular (positionally unbiased) attention operations. Conceptually, this connects positional bias, masks, and positional embeddings in a single formal framework and suggests an approach to efficient ALiBi-biased attention. Experiments on large language models validate our theoretical findings.
\end{abstract}

\section{Introduction}

Although the Transformer architecture has remained remarkably intact since its introduction in \cite{vaswani2017attention}, even through widespread adoption in large language models, encoding token positions stands out as a still evolving aspect. The challenge is to encode how the structural positions of tokens in the input inform the way they attend to each other. A closely intertwined challenge is long-context attention: attending to all tokens in a long context becomes computationally difficult or infeasible, whereas limiting attention to structurally local windows often degrades long-range reasoning performance. 

Three leading paradigms in positional encoding, often used in conjunction, are as follows:
\begin{itemize}
    \item \emph{Positional embedding}, where each position in the input is endowed with an embedding vector designed to encode its location -- e.g., sinusoidal \cite{vaswani2017attention} or rotary position embeddings (RoPE) \cite{su2024roformer};
    \item \emph{Attention with masks}, which allows tokens to attend to some tokens but not to others through a fixed binary mask -- e.g., sliding-window attention \cite{beltagy2020longformer,jelassi2024repeat};
    \item \emph{Attention with bias}, which modifies attention weights a posteriori according to token positions -- e.g., Attention with Linear Biases (ALiBi) \cite{press2021train}. 
\end{itemize}

In this work, we undertake a theoretical study of attention with positional bias, focusing on ALiBi, a prominent method that has emerged from empirical research \cite{press2021train} and has been widely implemented and adopted in popular models \cite{pli2024alibi,wuflashbias}. 
We study it through the lens of \emph{locality sensitive hashing} (LSH), a central paradigm in the theory of efficient algorithms. 
This leads us to uncover fundamental connections between ALiBi and LSH with implications to efficient long-context attention. 

Our main structural result, \Cref{thm:main}, shows that the ALiBi bias matrix can be approximated arbitrarily well by a linear combination of binary contiguous block-diagonal matrices. 
Conceptually, this draws a formal connection between the three main paradigms in positional encoding: bias, masks and embeddings. Namely, it suggests that attention with bias can be approximated by attention with masks, where the masks are induced by locality-sensitive hash collisions of positional embeddings. In the case of ALiBi, this connection holds in a strong formal sense. 

Our main algorithmic result, \Cref{thm:alg}, is an efficient approximation algorithm for ALiBi-biased attention over long contexts, which is based on our structural result.
It suggests an approach to long-context ALiBi through a reduction to a collection of randomized short-context regular (positionally unbiased) attention operations, which correspond to the contiguous binary blocks in \Cref{thm:main}. This yields a near-linear time algorithm for ALiBi with a provable approximation guarantee, addressing the computational bottleneck in long-context ALiBi, a challenge called out in \cite{press2021train}.

Conceptually, our contribution goes beyond applying classical LSH to positions in attention: it is in reinterpreting positional bias through hashing on positions, drawing a formal connection between positional embeddings, binary attention masks, and bias terms. Technically, while kernel approximation via LSH is standard, we introduce novel technical tools in order to prove matrix convergence in the spectral norm and high-probability block size control. These results are specialized to attention with ALiBi and do not follow from generic LSH machinery alone. Our proofs draw on results from matrix concentration theory, Fourier analysis of Toeplitz matrices, and sub-gamma tail bounds. 

We validate our theoretical findings through experiments on publicly available large language models.

\section{Our Results in Detail}\label{sec:results}


We start by fixing notation to be used throughout. 
We use ``$\odot$'' for element-wise matrix product.
We denote the following norms for a matrix $M$: $\norm{M}$ for the spectral norm; $\norm{M}_{\max}=\max_{ij}|M_{ij}|$;  
and $\norm{M}_{2,\infty}=\max_i\sqrt{\sum_jM^2_{ij}}$. 
We let $D^{[M]}$ denote the diagonal row-sum matrix $D^{[M]}_{ii}=\sum_{j=1}^nM_{ij}$.
Finally, we will use $J\in\R^{n\times n}$ for the fixed lower-triangular all-1 matrix, $J_{ij}=\mathbf1\{j\leq i\}$. 



\textbf{Attention.}
Let $n$ be the attention context length. Let $Q,K\in\R^{n\times d}$ be the key and query matrices with rows $\{q_i\},\{k_j\}$ respectively. The unnormalized dot-product attention matrix $A\in\R^{n\times n}$ has entries 
$A_{ij}=\exp(q_i^Tk_j/\sqrt d)$.
The row-wise normalized non-causal attention matrix is $P=(D^{[A]})^{-1}A$. 
In the causal attention case we instead have $P=(D^{[A\odot J]})^{-1}(A\odot J)$. 
Given a value matrix $V\in\R^{n\times d'}$, the output of the attention operation without positional bias is $T=PV$.

\textbf{ALiBi.}
The ALiBi matrix $\alibi$ is defined in \cite{press2021train} as $\alibi_{ij} = e^{-|i-j|/\sigma}$, where $\sigma>0$ is a bandwidth hyperparameter which is fixed per head and varies between heads. 
The unnormalized ALiBi-biased attention weights are $A^\star = A\odot\alibi$ in the non-causal case and $A^\star = A\odot J\odot\alibi$ in the causal case. 
In either case, the normalized ALiBi-biased attention matrix is $P^\star=(D^{[ A^\star]})^{-1}A^\star$, and the goal is to compute the output $T^\star = P^\star V$. 

The following theorem is our main structural result: an approximation of the ALiBi matrix by random contiguous block-diagonal binary matrices. 

\begin{theorem}\label{thm:main}
There exists a distribution $\mathcal M$ over contiguous block-diagonal binary $n\times n$ matrices, where each diagonal block is a square all-1 matrix, such that $\E_{M\sim\mathcal M}[M]=\alibi$. 
Let $M_1,\ldots,M_s\sim\mathcal M$ be i.i.d.~samples and let $\widetilde M=\frac1s\sum_{i=1}^sM_i$ be their empirical mean. Then we have, 
\begin{itemize}
\item Spectral norm expectation:
\begin{equation}\label{eq:spectral_expectation}
  \E\norm{\alibi-\widetilde M} \leq 
  C\Psi_\sigma\left(\sqrt{\frac{\log n}{s}} + \frac{\log n}{s}\right)
\end{equation}
\item Spectral norm concentration:
\begin{equation}\label{eq:spectral_concentration}
  \forall\varepsilon>0,\;\;\Pr[\norm{\alibi-\widetilde M} \geq \varepsilon] \leq  n\cdot\exp\left(-Cs\min\left\{\frac{\varepsilon^2}{\Psi_\sigma^2}, \frac{\varepsilon}{\Psi_\sigma}\right\}\right)
\end{equation}
\item Max-norm concentration:
\begin{equation}\label{eq:maxnorm_concentration}
  \forall\varepsilon>0,\quad\Pr[\norm{\alibi-\widetilde M}_{\max} \geq \varepsilon] \leq  2n^2e^{-2s\varepsilon^2}
\end{equation}
\end{itemize}
where $\Psi_{\sigma} = 1 + \sigma + \frac{e^{1/\sigma}+1}{e^{1/\sigma}-1}$ and $C>0$ is a universal constant. 
Furthermore, let $b_{\max}$ be the maximum size of a diagonal block in any of the samples $M_1,\ldots,M_s$, then, for all $\delta\in(0,1/e)$,
\begin{equation}\label{eq:thm_blocksize} 
        \Pr[b_{\max} > 1+ \sigma(3\ln(s/\delta) +2)] < \delta . 
        \end{equation}
\end{theorem}
Structurally, the theorem states that the ALiBi matrix is a convex combination of contiguous non-overlapping local-window attention patterns, each positionally unbiased. 
Approximation-wise, it implies that with high probability, $s\approx\varepsilon^{-2}\sigma^{2}\log n$ samples from $\mathcal M$ suffice to approximate ALiBi using diagonal blocks of size at most $b_{\max}\approx\sigma\log\log n$ (suppressing logarithmic factors in $\varepsilon^{-1},\delta^{-1},\sigma$ for simplicity). Regarding the values of $\sigma$ and $\Psi_\sigma$, we recall that Press et al.~\cite{press2021train} set $\sigma$ as a geometric series from $2^{8/n_h}$ to $2^8$ across $n_h$ heads. Under this setting, $\Psi_\sigma$ satisfies $2<\Psi_\sigma<3\sigma+2$. 

Algorithmically, \Cref{thm:main} suggests that ALiBi-biased attention can be approximated efficiently over long contexts by a reduction to a collection of small local attention windows. 
Our next theorem formalizes this intuition and yields an approximation guarantee for the output attention weights. 



\begin{theorem}[efficient uniform approximation for ALiBi-biased attention]\label{thm:alg}
Let $\delta\in(0,1/e)$ and $s\geq\log(n/\delta)$.
Let $\widetilde M=\frac1s\sum_{i=1}^sM_i$ be the empirical mean of $s$ samples from $\mathcal M$ as in \Cref{thm:main}.  
Given an attention instance $(K,Q,V)$, let $A,A^\star,P$ be the corresponding attention weights as above (in either the causal or non-causal variant), and denote:
\[
  \Delta_P = C\min\{\Psi_\sigma\norm{P}_{2,\infty},1\}\sqrt{\frac{\log(n/\delta)}{s}} \quad \text{and} \quad \beta_P^\star=\min_i\sum_{j=1}^nP_{ij}\alibi_{ij} 
\]
where $C>0$ is a universal constant.
Then, with probability $1-\delta$ over the draw of $\widetilde M$, the following hold simultaneously for all input attention instances:
\begin{enumerate}
    \item The approximate ALiBi-biased attention output $\widetilde T^\star$, obtained by replacing $L^\star$ with $\widetilde M$, can be computed in time $O(nd(\sigma s\log(s/\delta)+1))$. 
    \item Whenever $\Delta_P<\beta^\star_P$, it holds that
    $\norm{T^\star - \widetilde T^\star}_{\max} \leq \frac{2\Delta_P}{\beta_P^\star-\Delta_P}\norm{V}_{\max}$.
\end{enumerate}
\end{theorem}
To unpack the statement, we note that the critical quantity $\beta_P^\star$ is the ratio between the minimum row denominator in $A^\star$ and $A$, that is, in the attention matrix before and after the ALiBi bias. It governs how sensitive the instance is to perturbations in the bias matrix. 
The term $\Delta_P$ which governs the error decays like $1/\sqrt s$ with the number of samples $s$, and may further improve with smaller $\norm{P}_{2,\infty}$, which occurs in instances where attention weights are more balanced. 
Importantly, the convergence in matrix norms in \Cref{thm:main} (\cref{eq:spectral_concentration,eq:maxnorm_concentration}) is what allows the approximation in \Cref{thm:alg} to hold uniformly --- that is, simultaneously for all attention inputs with high probability. 



\section{Positional LSH}
\begin{figure}
     \includegraphics[width=\textwidth]{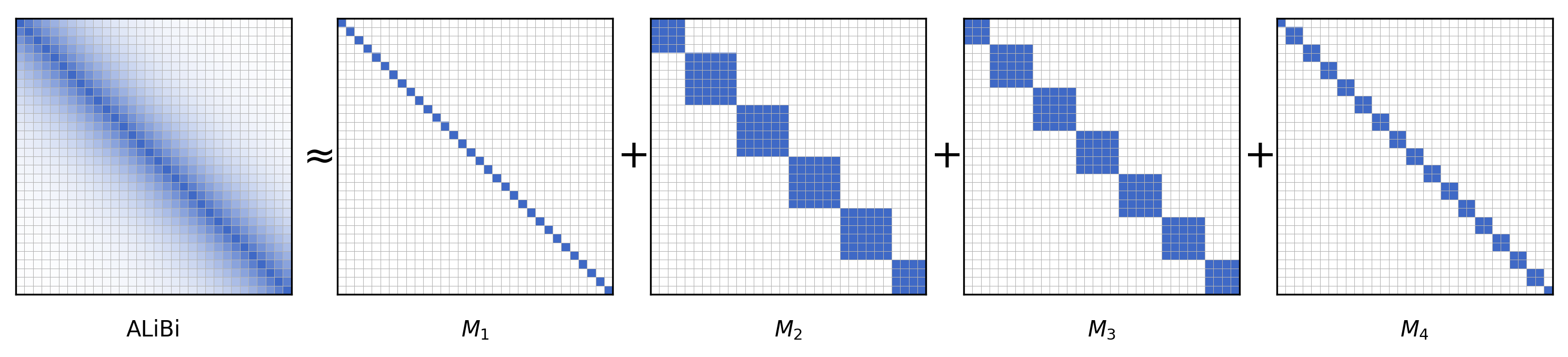}
 \caption{ALiBi approximation with binary block-diagonal matrices by positional LSH.}\label{fig:plsh}
\end{figure}

 In this section we outline our construction and analysis. 

\subsection{Attention with Bias via Locality Sensitive Hashing}\label{sec:plsh_generic}
 
Locality sensitive hashing (LSH) was originally introduced in \cite{indyk1998approximate} for metric spaces, with the application of fast nearest neighbor search. 
A different but closely related form of LSH can be defined for kernel spaces. It originates in \cite{broder1997resemblance,charikar2002similarity,andoni2008near} and is useful for efficiently estimating similarities for purposes beyond search. This is the notion of LSH we use in this paper. 

Let $U$ be a space of objects and $\kk:U\times U\rightarrow[0,1]$ a similarity map (i.e., a kernel) over $U$. An LSH scheme for $\kk$ is a distribution $\mathcal H$ over hash functions operating on $U$ such that for all $u,u'\in U$
\begin{equation}\label{eq:lsh}
    \kk(u,u')=\Pr_{h\sim\mathcal H}[h(u)=h(u')] . 
\end{equation}
This notion has been widely influential, including recent applications to efficient attention computation over long contexts (see \Cref{sec:related}).
While prior work has applied LSH to the token embeddings of the keys and queries in the attention operation, here we will apply it exclusively to their positions in the sequence. Hence, we call this \emph{positional LSH}.

Let $L$ be a positional bias matrix that biases the original attention weights $A$ as $A\odot L$. Let $\{u_i\}_{i=1}^n$ be positional embeddings for the token positions in a context of length $n$. That is, $u_i$ is an object that encodes the location $i$ within the context. Suppose that $L$ has the form $L_{ij}=\kk(u_i,u_j)$. Furthermore, suppose that $\kk$ admits an LSH scheme $\mathcal H$. 
We may define a distribution $\mathcal M$ over $n\times n$ matrices where a sampled hash function $h\sim\mathcal H$ induces the random binary matrix $M^{(h)}$ given by
\[ M^{(h)}_{ij} = \begin{cases}
    1 & \text{if $h(u_i)=h(u_j)$} \\
    0 & \text{if $h(u_i)\neq h(u_j)$}.
\end{cases}\]
The LSH property (\ref{eq:lsh}) immediately implies $L=\E_{h\sim\mathcal H}[M^{(h)}]$.
Furthermore, $M^{(h)}$ has a distinct structure: its 1-entries are organized in non-overlapping principal combinatorial rectangles. Specifically, for each hash bin $\ell$ in the range of $h$, denote by $I_\ell=\{i:h(u_i)=\ell\}$ the indices that fall into that bin. 
$M^{(h)}$ has 1's in its $I_\ell\times I_\ell$ principal submatrix for every $\ell$, and 0's elsewhere. Equivalently, the rows and columns of $A$ can be permuted to make $M^{(h)}$ a block-diagonal matrix in which all blocks are all-1 submatrices. Note that without a permutation, the blocks $I_\ell$ need not necessarily be contiguous; it could be, for example, that $h(u_1)=h(u_3)\neq h(u_2)$, in which case indices $1$ and $3$ are hashed into the same bin while $2$ is hashed into a different bin. 

The identity $L=\E_{h\sim\mathcal H}[M^{(h)}]$ naturally suggests that $L$ can be approximated by the empirical average of i.i.d.~samples of $M^{(h)}$. See \Cref{fig:plsh} for illustration. 
This raises the question of the convergence rate to the mean. Entry-wise scalar convergence follows immediately from standard Chernoff-Hoeffding concentration, and this already yields the max-norm bound \cref{eq:maxnorm_concentration} in \Cref{thm:main}.
Matrix convergence in the spectral norm (\cref{eq:spectral_expectation,eq:spectral_concentration} in \Cref{thm:main}) is considerably more intricate and may not hold in general (see \Cref{sec:generality}). We will prove it for ALiBi with a specific positional LSH scheme. 

\subsection{Positional LSH for ALiBi via Random Binning Features}\label{sec:rbf}
ALiBi biases attention weights according to the similarity measure $L^{\star}_{ij}=\kk(i,j)=e^{-|i-j|/\sigma}$ between token positions $i,j$. This is the familiar one-dimensional Laplacian kernel. 
Note that the positional embedding object of each index $i$ is simply its integer value, $u_i=i$. 

LSH schemes for the Laplacian kernel are well-known \cite{rahimi2007random,andoni2009dimension,backurs2019space}. 
We will use the one due to Rahimi and Recht \cite{rahimi2007random}, which they termed Random Binning Features.

\begin{definition}[\cite{rahimi2007random}]\label{def:rbf}
The Random Binning Features (RBFs) LSH scheme over $\R$ is defined as follows. 
To sample $h\sim\mathcal H$, first draw $b\sim\Gamma(2,\sigma)$ from the Gamma distribution with shape 2 and scale $\sigma$. 
It is supported on $[0,\infty)$ with density $p(b)=\sigma^{-2}b e^{-b/\sigma}$. Then sample $c$
uniformly from $[0,b]$. The hash function $h=h_{b,c}$ partitions $\R$ into length-$b$ segments $\{[b\ell+c,b\ell+b+c):\ell\in\Z\}$ and hashes each number into the segment that contains it. Formally, 
$h(u)=\lfloor\frac{u-c}{b}\rfloor$ for every $u\in\R$.  
\end{definition}

\begin{lemma}[\cite{rahimi2007random}]\label{lmm:rbf}
    RBFs satisfy \cref{eq:lsh} for the Laplacian kernel $\kk(u,u')=e^{-|u-u'|/\sigma}$ over $\R$.
\end{lemma}

Therefore, RBFs form a positional LSH scheme for ALiBi. 
Although RBFs for the Laplacian kernel are classical, our use of them will not be a black-box application: the nontrivial step will be to analyze their interaction particularly with the space of token indices and achieve matrix approximation for ALiBi-biased attention. 
To this end, we will require specific properties of their underlying Gamma distribution, summarized in the following fact.
\begin{fact}\label{fct:gamma}
A Gamma random variable $Z\sim\Gamma(2,\sigma)$ has moments $\E[Z^k]=\sigma^k(k+1)!$ for all $k\in\N$ and moment generating function $\E[e^{tZ}]=(1-\sigma t)^{-2}$ for all $t<1/\sigma$. 
\end{fact}

We observe some additional useful properties of RBFs as a positional LSH scheme over sequence indices $\{1,\ldots,n\}$. 
First, the hash bins it forms are contiguous -- the indices in each bin are consecutive in the context. Therefore, the diagonal all-1 blocks are contiguous in $M^{(h)}$ as defined in the previous section. Second, the hash bins have nearly uniform sizes: each bin contains exactly $\lceil b \rceil$ consecutive indices, except for the first and last bins, which may contain fewer indices due to being on the boundary of the context. 

\subsection{Approximation in Matrix Norms}\label{sec:matrixnorms}
We turn to proving \cref{eq:spectral_expectation,eq:spectral_concentration,eq:maxnorm_concentration} in \Cref{thm:main}. As mentioned earlier, \cref{eq:maxnorm_concentration} follows from standard scalar concentration for Bernoullis (see \Cref{lmm:maxnorm_app}). 

\begin{lemma}\label{lmm:maxnorm_app}
    Under the assumptions of \Cref{thm:main}, the max-norm bound in \cref{eq:maxnorm_concentration} holds.
\end{lemma}
\begin{proof}
    Let $i,j\in[n]$. By \Cref{lmm:rbf} and the positional LSH construction in \Cref{sec:plsh_generic}, we have $L_{ij}^\star=\Pr_{M\sim\mathcal M}[M_{ij}=1]$. Therefore, $\widetilde M_{ij}$ is the average of $s$ Bernoulli random variables with expectation $L_{ij}^\star$. Therefore, by the Chernoff-Hoeffding inequality, 
    \[
      \Pr\left[\left|L_{ij}^\star - \widetilde M_{ij}\right| > \varepsilon\right] < 2e^{-2s\varepsilon^2}
    \]
    The lemma follows by a union bound over all $i,j$. 
\end{proof}

For the spectral norm guarantees, \cref{eq:spectral_expectation,eq:spectral_concentration}, we will use the following subexponential matrix Bernstein inequality with unbounded moments. 
\begin{theorem}[\cite{tropp2012user}, Theorem 6.2 and Remark 6.5]\label{thm:matrix_bernstein}
    Let $X_1,\ldots,X_s$ be i.i.d.~samples from a distribution $X$ over symmetric $n\times n$ matrices with $\E[X]=L$. Suppose there are $r,\alpha>0$ such that 
    \begin{equation}\label{eq:bernstein_moment_condition}
      \forall\; k\in\N, \quad \norm{\E[(X-L)^k]} \leq \frac{k!}{2} \cdot r^{k-2} \cdot \alpha^2 .
    \end{equation}
    Then, the empirical mean $\widetilde X=\frac1s\sum_{i=1}^sX_i$ satisfies
    $\E[\norm{L - \widetilde X}] \leq \frac{O(1)}{s}\cdot \max\{\alpha\sqrt{s\log n}, r\log n\}$, 
    and $\Pr[\norm{L - \widetilde X} \geq \varepsilon] \leq n\cdot\exp\left(-\frac{s\cdot\varepsilon^2}{2\alpha^2+2r\varepsilon}\right)$ for all $\varepsilon>0$. 
\end{theorem}
We thus need to establish \cref{eq:bernstein_moment_condition}
for $M\sim\mathcal M$ with appropriate $r,\alpha$. Fixing $k\in\N$, we have by the sub-additivity and sub-multiplicativity of the spectral norm, and a binomial expansion,
\begin{equation}\label{eq:error_moments}
      \norm{\E[(M-\alibi)^k]} \leq \E\left[\left(\norm{M}+\norm{\alibi}\right)^k\right] = \sum_{j=0}^k{k\choose j}\E[\left(\norm{M}-1\right)^j]\left(\norm{\alibi}+1\right)^{k-j}.
\end{equation}
Thus we need to bound the moments $\E[\left(\norm{M}-1\right)^j]$ and the powers $\left(\norm{\alibi}+1\right)^{j}$. 

\begin{lemma}\label{lmm:lsh_matrix_moments}
    $\E[\left(\norm{M}-1\right)^k]\leq \sigma^k(k+1)!$ for all $k\in\N$.
\end{lemma}
\begin{proof}
    $M$ is a block diagonal matrix with all-1 blocks. Let $b_M$ denote the size of its largest block. 
    It is straightforward to see that the largest eigenvalue of $M$ is $b_M$, thus $\norm{M}=b_M$. 
    Next, recall that $b_M\leq \lceil b \rceil\leq b+1$ where $b\sim\Gamma(2,\sigma)$. Hence, $\left(\norm{M}-1\right)^k\leq b^k$. Taking expectation of both sides, the lemma follows from the formula for Gamma moments in \Cref{fct:gamma}. 
\end{proof}

\begin{lemma}\label{lmm:alibi_spectral_norm}
    $\norm{\alibi}\leq (e^{1/\sigma}+1)/(e^{1/\sigma}-1)$. 
\end{lemma}
\begin{proof}
    We show two ways to prove the lemma. The first is a more generic approach based on Fourier analysis of Toeplitz matrices, which has the potential advantage of possibly generalizing better to other positional bias matrices beyond ALiBi. 
Recall that a matrix $T$ is Toeplitz if there is a function $g:\Z\rightarrow\R$ such that $T_{ij}=g(i-j)$ for all $i,j$. The ALiBi matrix $\alibi$ is Toeplitz with $g(z)=e^{-|z|/\sigma}$. 
The following lemma offers a way to bound the spectral norm of a Toeplitz matrix.  
    \begin{lemma}[\cite{gray2006toeplitz}, Lemma 4.1]\label{lmm:toeplitz_spectral}
        Let $T$ be a Toeplitz matrix with $T_{ij}=g(i-j)$. Suppose $g$ satisfies $\sum_{i=-\infty}^{\infty}|g(i)|<\infty$. Then $\norm{T} \leq \sup f$ where $f$ is the discrete-time Fourier transform (DTFT) of $g$. 
    \end{lemma}
    Thus, to bound the spectral norm of $\alibi$ with \Cref{lmm:toeplitz_spectral} we need the DTFT of $g(z)=e^{-|z|/\sigma}$. 
    A standard fact, which can be verified by direct derivation, is that the DTFT of $g_\phi(z)=\phi^{|z|}$ for $\phi\in(-1,1)$ is $f(\omega)=(1-\phi^2)/(1-2\phi\cos(\omega)+\phi^2)$ and its supremum is $(1+\phi)/(1-\phi)$ (attained at $\cos(\omega)=1$). For $\alibi$ we have $\phi=e^{-1/\sigma}$, and thus \Cref{lmm:alibi_spectral_norm} follows from \Cref{lmm:toeplitz_spectral}.
 
The second proof for \Cref{lmm:alibi_spectral_norm} is more direct, albeit more specialized to ALiBi. Recall the matrix norm definitions $\norm{M}_1=\max_j\sum_i|M_{ij}|$ and $\norm{M}_\infty=\max_i\sum_j|M_{ij}|$. By a standard H\"older inequality for matrix norms, $\norm{\alibi}\leq\sqrt{\norm{\alibi}_1\norm{\alibi}_\infty}$. Since $\alibi$ is symmetric, $\norm{\alibi}_1=\norm{\alibi}_\infty$, hence $\norm{\alibi}\leq\norm{\alibi}_\infty$. Thus, it suffices to bound the row sums in $\alibi$. The sum of row $i$ is
\[
  \sum_{j=1}^n\alibi_{ij}
  = 1 + \sum_{t=1}^{i-1}e^{-t/\sigma} +  \sum_{t=1}^{n-i}e^{-t/\sigma} \leq 1  + 2\sum_{t=1}^{\infty}e^{-t/\sigma} = \frac{e^{1/\sigma}+1}{e^{1/\sigma}-1},
\]
where the final equality is established by calculating the sum of the geometric series.
\end{proof}

We plug \Cref{lmm:lsh_matrix_moments,lmm:alibi_spectral_norm} into \cref{eq:error_moments} and get
\[
\norm{\E[(M-\alibi)^k]} 
\leq \sum_{j=0}^k{k\choose j}\cdot 
\sigma^j(j+1)! \cdot \left(\frac{e^{1/\sigma}+1}{e^{1/\sigma}-1} + 1 \right)^{k-j} 
\leq 2^kk!\sum_{j=0}^k{k\choose j}\cdot 
\sigma^j \cdot \left(\frac{e^{1/\sigma}+1}{e^{1/\sigma}-1}+1\right)^{k-j}
\]
\[
= 2^kk!\cdot\left(1+\sigma + \frac{e^{1/\sigma}+1}{e^{1/\sigma}-1}\right)^k = (2\Psi_\sigma)^kk! = \frac{k!}{2}\cdot(2\Psi_\sigma)^{k-2}\cdot(2\sqrt2\Psi_\sigma)^2,
\]
where we have used that $(j+1)!\leq(k+1)!\leq 2^kk!$ for all $k\geq j\geq0$. Thus, \cref{eq:bernstein_moment_condition} holds with $r=2\Psi_\sigma$ and $\alpha=2\sqrt2\Psi_\sigma$, and thus \cref{eq:spectral_expectation,eq:spectral_concentration} in \Cref{thm:main} now follow from \Cref{thm:matrix_bernstein}. 

\subsection{Maximal Block Size}\label{sec:blocksize}
We proceed to proving \cref{eq:thm_blocksize} in \Cref{thm:main}.
By the previous sections, this amounts to bounding the maximal bin size in RBFs. 
To this end, we recall subgamma concentration (see, e.g., \cite{waggoner2019subgamma}). \begin{definition}
A real-valued random variable $Z$ is called \emph{$(\nu,\kappa)$-right-subgamma} if for all $t\in(0,1/\kappa)$ it holds that 
$\E\left[e^{t(Z-\E[Z])}\right] \leq \exp\left(\nu t^2/(2-2\kappa t)\right)$.
\end{definition}
\begin{fact}\label{fct:rightsubgamma}
    if $Z$ is $(\nu,\kappa)$-right-subgamma, then for all $t>0$,
    $\Pr[Z\geq\E[Z] + \sqrt{2\nu t} + \kappa t]\leq e^{-t}$.
\end{fact}
\begin{lemma}\label{clm:gammarightsubgamma}
    $Z\sim\Gamma(2,\sigma)$ is $(2\sigma^2,\sigma)$-right-subgamma.
\end{lemma}
\begin{proof}
    Let $t\in(0,1/\sigma)$. By plugging the expectation and MGF from \Cref{fct:gamma},
    \[
        \ln\E[e^{t(Z-\E[Z])}] = \ln\left(\E[e^{tZ}]\cdot e^{-t\E[Z]}\right) 
        = \ln((1-\sigma t)^{-2} \cdot e^{-2\sigma t}) 
        = 2(-\ln(1-\sigma t)-\sigma t).
    \]
    From the Taylor expansion of $-\ln(1-x)$ at $0$ we have, for every $x\in(0,1)$,
    \[
      -\ln(1-x) - x
      = \sum_{i=1}^\infty\frac{x^i}{i} - x
      = \sum_{i=2}^\infty\frac{x^i}{i}
      \leq \frac{x^2}{2}\sum_{i=0}^\infty x^i 
      = \frac{x^2}{2} \cdot \frac{1}{1-x}.
    \]
    Plugging this above with $x=\sigma t$,
    \[ \ln\E[e^{t(Z-\E[Z])}] \leq 2\cdot\frac{\sigma^2t^2}{2}\cdot\frac{1}{1-\sigma t} 
    = \frac{2\sigma^2t^2}{2(1-\sigma t)}.  \]
    Exponentiating both sides proves the claim.
\end{proof}

In \Cref{thm:main} we draw i.i.d. samples $M_1,\ldots,M_s\sim\mathcal M$ and denote by $b_{\max}$ their maximal block size. Let $b_1,\ldots,b_s\sim\Gamma(2,\sigma)$ be the corresponding samples of their RBFs. By the previous sections, each $M_j$ is block-diagonal with maximal bin size $\lceil b_j\rceil\leq b_j+1$. Therefore, $b_{\max}\leq 1+\max_jb_j$. 
We use \Cref{fct:rightsubgamma} with $t=\ln(s/\delta)$, $\E[b]=2\sigma$ from \Cref{fct:gamma}, and $\nu,\kappa$ from \Cref{clm:gammarightsubgamma}. We get
\[ 
  \forall j, \quad
  \Pr[b_j > 1+ 2\sigma + 2\sigma\sqrt{\ln(s/\delta)} + \sigma \ln(s/\delta)] < \delta/s .
\]
Since $s\geq1$ and $\delta<1/e$, \cref{eq:thm_blocksize} follows by a union bound over $j=1,\ldots,s$.

\section{Near-Linear time Approximation Algorithm}\label{sec:alg}
In this section we prove \Cref{thm:alg}. 
The algorithm in the theorem replaces the exact ALiBi matrix $\alibi$ with the approximation $\widetilde M$. 
Given an attention instance $(Q,K,V)$, in the notation from \Cref{sec:results}, we have $A^\star=A\odot\alibi$ (non-causal case) or $A^\star=A\odot J\odot\alibi$ (causal case), $P^\star=(D^{[A^\star]})^{-1}A^\star$, and the exact output is $T^\star=P^\star V$. The algorithm in \Cref{thm:alg} computes $\widetilde A^\star=A\odot\widetilde M$ (non-causal case) or $\widetilde A^\star=A\odot J\odot\widetilde M$ (causal case), then $\widetilde P^\star=(D^{[\widetilde A^\star]})^{-1}\widetilde A^\star$, and returns $\widetilde T^\star=\widetilde P^\star V$.

\paragraph{Running time analysis.} 
Let $M$ be a sample from $\mathcal M$ and let $I_1,\ldots,I_t\subset\{1,\ldots,n\}$ be its contiguous blocks of indices. 
Let $\widetilde A=A\odot M$. 
Since $M$ and thus $\widetilde A$ is block-diagonal, we can compute $\widetilde AV$ block by block as $\widetilde AV=\sum_{j=1}^t\widetilde A_{[I_j,I_j]}V_{[I_j,:]}$. 
Since the blocks in $M$ are all-1 matrices, $\widetilde A_{[I_j,I_j]}=(A\odot M)_{[I_j,I_j]}=A_{[I_j,I_j]}$, and hence  $\widetilde AV=\sum_{j=1}^t A_{[I_j,I_j]}V_{[I_j,:]}$. Similarly, $D^{[\widetilde A]}$ can be computed as $D^{[\widetilde A]}=\sum_{j=1}^tD^{[A_{[I_j,I_j]}]}$. The upshot is that $\widetilde A$ and $D^{[\widetilde A]}$ can be computed by computing only the restriction of $A$ to the blocks of $M$. Note that $A$ is the regular (positionally unbiased) attention matrix, and therefore, computing $\widetilde A$ and $D^{[\widetilde A]}$ is reduced to a collection of smaller regular attention computations on the subsets of keys and queries that correspond to the blocks $\{I_j\}$. The time complexity is $d\sum_{j=1}^t|I_j|^2$. 
Observe that $\sum_{j}|I_j|=n$ since the blocks form a partition of $\{1,\ldots,n\}$, and therefore $d\sum_{j=1}^t|I_j|^2\leq d\max_j|I_j|\sum_j|I_j|=nd\max_j|I_j|$.  
With the empirical mean $\widetilde M=\tfrac1s\sum_{i=1}^sM_i$ of $s$ samples with block sizes $\{I_{ij}\}$, the overall time complexity for computing $\widetilde T^*$ is thus $\sum_{i=1}^snd\max_j|I_{ij}|\leq ndsb_{\max}$. The probabilistic bound on $b_{\max}$ in \Cref{thm:main} (\cref{eq:thm_blocksize}) now establishes the running time in \Cref{thm:alg} in the non-causal case. 
The same arguments hold in the causal case with $A\odot J$ instead of $A$.  

\paragraph{Approximation guarantee.} 
Let $A'=A$ in the non-causal case and $A'=A\odot J$ in the causal case. In either case we have $P=(D^{[A']})^{-1}A'$, and $P$ is row-stochastic. 

Observe that the exact target output $T^\star$ and the returned approximation $\widetilde T^\star$ are computed as
\[
T_{ij}^\star = \frac{\sum_{k=1}^nA_{ik}'L_{ik}^\star V_{kj}}{\sum_{k=1}^nA_{ik}'L_{ik}^\star} \quad \text{and} \quad \widetilde T_{ij}^\star = \frac{\sum_{k=1}^nA_{ik}'\widetilde M_{ik} V_{kj}}{\sum_{k=1}^nA_{ik}'\widetilde M_{ik}} .
\]
We may divide both the numerators and the denominators by $D^{[A']}_{ii}$ and write, 
\[
T_{ij}^\star = \frac{\sum_{k=1}^nP_{ik}L_{ik}^\star V_{kj}}{\sum_{k=1}^nP_{ik}L_{ik}^\star} \quad \text{and} \quad \widetilde T_{ij}^\star = \frac{\sum_{k=1}^nP_{ik}\widetilde M_{ik} V_{kj}}{\sum_{k=1}^nP_{ik}\widetilde M_{ik}} .
\]
The error in the numerators can be bounded either with the max-norm of $L^\star-\widetilde M$ as
\begin{equation}\label{eq:maxnormbound}
\left|\sum_kP_{ik}(L^\star-\widetilde M)_{ik}V_{kj}\right|
\leq \norm{L^\star-\widetilde M}_{\max}\norm{V}_{\max}\left|\sum_kP_{ik}\right| = \norm{L^\star-\widetilde M}_{\max}\norm{V}_{\max},
\end{equation}
or with the spectral norm of $L^\star-\widetilde M$ as
\begin{equation}\label{eq:spectralnormbound}
\left|\sum_kP_{ik}(L^\star-\widetilde M)_{ik}V_{kj}\right|
= \left|(L^\star-\widetilde M)_{i,*}^\top (P_{i,*}\odot V_{*,j})\right| 
\leq \norm{L^\star-\widetilde M}\cdot \norm{V}_{\max}\norm{P}_{2,\infty} .
\end{equation}
The same holds for the denominators with an all-1 matrix instead of $V$:
\begin{equation}\label{eq:maxnormbound_denom}
\left|\sum_kP_{ik}(L^\star-\widetilde M)_{ik}\right|
\leq \norm{L^\star-\widetilde M}_{\max}\left|\sum_kP_{ik}\right| = \norm{L^\star-\widetilde M}_{\max},
\end{equation}
and
\begin{equation}\label{eq:spectralnormbound_denom}
\left|\sum_kP_{ik}(L^\star-\widetilde M)_{ik}\right|
= \left|(L^\star-\widetilde M)_{i,*}^\top P_{i,*}\right| 
\leq \norm{L^\star-\widetilde M}\cdot \norm{P}_{2,\infty} .
\end{equation}
For convenience let us denote the numerators and the denominators as
\[
N_{ij} = \sum_{k=1}^nP_{ik}L_{ik}^\star V_{kj}, \quad 
\widetilde N_{ij} = \sum_{k=1}^nP_{ik}\widetilde M_{ik} V_{kj}, \quad 
D_i = \sum_{k=1}^nP_{ik}L_{ik}^\star,  \quad \widetilde D_i = \sum_{k=1}^nP_{ik}\widetilde M_{ik} .
\]
\Cref{eq:maxnormbound,eq:spectralnormbound} yield
\[
  \forall i,j\quad\left|N_{ij}-\widetilde N_{ij}\right| \leq \min\{\norm{L^\star - \widetilde M}_{\max} , \norm{P}_{2,\infty}\norm{L^\star - \widetilde M}\}\cdot\norm{V}_{\max},
\]
and \Cref{eq:maxnormbound_denom,eq:spectralnormbound_denom} yield
\[
  \forall i\quad\left|D_{i}-\widetilde D_{i}\right| \leq \min\{\norm{L^\star - \widetilde M}_{\max} , \norm{P}_{2,\infty}\norm{L^\star - \widetilde M}\}.
\]
With probability $1-\delta$ we have by \Cref{eq:maxnorm_concentration,eq:spectral_concentration} in \Cref{thm:main},
\[
  \norm{L^\star - \widetilde M}_{\max} \leq O\left(\sqrt{\frac{\log(n/\delta)}{s}}\right) \quad \text{and} \quad 
  \norm{L^\star - \widetilde M} \leq O\left(\Psi_\sigma\sqrt{\frac{\log(n/\delta)}{s}} + \frac{\log(n/\delta)}{s}\right).
\]
In the statement of \Cref{thm:alg} we assume that $s\geq\log(n/\delta)$ and hence $\frac{\log(n/\delta)}{s} \leq \sqrt{\frac{\log(n/\delta)}{s}}$. Plugging this above, we get
\begin{equation}\label{eq:ndbound}
 \forall i,j\;\;\left|N_{ij}-\widetilde N_{ij}\right| \leq \Delta_P\norm{V}_{\max} \quad \text{and} \quad \forall i\;\; \left|D_{i}-\widetilde D_{i}\right| \leq \Delta_P . 
\end{equation}
Furthermore, we have
\begin{equation}\label{eq:tbound}
  \forall i,j \quad \left|\frac{N_{ij}}{D_i}\right| =  \left|\frac{\sum_{k=1}^nP_{ik}L_{ik}^\star V_{kj}}{\sum_{k=1}^nP_{ik}L_{ik}^\star}\right|
  \leq \left|\frac{\sum_{k=1}^nP_{ik}L_{ik}^\star}{\sum_{k=1}^nP_{ik}L_{ik}^\star}\right| \cdot\norm{V}_{\max} = \norm{V}_{\max} .
\end{equation}
From \Cref{eq:ndbound,eq:tbound},
\begin{equation}\label{eq:almostthmalg}
  \left|\widetilde T_{ij}^\star - T_{ij}^\star\right|
  = \left|\frac{\widetilde N_{ij}}{\widetilde D_i} - \frac{N_{ij}}{D_i}\right|
  = \left|\frac{\widetilde N_{ij} - N_{ij}}{\widetilde D_i} + \frac{N_{ij}}{D_i}\cdot\frac{D_i - \widetilde D_i}{\widetilde D_i}\right|
 \leq \frac{2\Delta_P\norm{V}_{\max}}{\widetilde D_i} .
\end{equation}
Finally, observe that $\beta_P^\star=\min_iD_i$ in the notation of \Cref{thm:alg}. Therefore, if the condition of item 2 in the theorem holds, then by \Cref{eq:ndbound} we have for all $i$ that $\widetilde D_i \geq D_i - \left|\widetilde D_i - D_i\right| \geq \beta_P^\star - \Delta_P$. 
Plugging this in \Cref{eq:almostthmalg} yields the theorem.

\section{Bounding the Number of HBM Accesses}\label{sec:hbm}

\Cref{thm:alg} bounds the time complexity of our algorithm. On modern GPUs, however, the running time of attention is often dominated by the cost of moving data between memory levels, rather than by its time complexity. In this section we analyze our algorithm in a model that captures this cost, following the IO-aware analysis of FlashAttention \cite{dao2022flashattention}.

A GPU has two main levels of memory. Its main memory, the HBM, is large (tens of gigabytes) and stores the model parameters and the inputs $Q,K,V$. Computation, however, takes place in the on-chip SRAM, which is much smaller but roughly an order of magnitude faster \cite{dao2022flashattention}. Every computation therefore requires moving its inputs from HBM to SRAM and its outputs back to HBM.
Since the compute throughput of GPUs has grown faster than their memory bandwidth, standard attention is largely memory-bound: its running time is governed mainly by the number of HBM accesses. For this reason, FlashAttention \cite{dao2022flashattention} measures the cost of attention by the number of HBM accesses, and reduces it by tiling the computation so that the full $n\times n$ attention matrix is never written to HBM.

Let $M$ denote the SRAM size, which satisfies $d \leq M \leq nd$. FlashAttention requires $\Theta(n^2d^2M^{-1})$ HBM accesses \cite{dao2022flashattention}. In our algorithm, each sample $M_r\sim\mathcal M$ is a contiguous block-diagonal binary matrix, so the attention computation for $M_r$ reduces to independent local attention operations over its blocks, each of which can be carried out with FlashAttention. This yields the following bound.

\begin{proposition}\label{prop:hbm}
Assume $d \leq M$. Then the expected number of HBM accesses incurred by the Positional LSH
attention algorithm of \Cref{thm:alg} is $O\!\left(snd\!\left(1 + d\sigma/M\right)\right)$.
\end{proposition}

Note that for $s \approx \varepsilon^{-2}\sigma^{2}\log n$, as implied by \Cref{thm:main} for an
$\varepsilon$-accurate approximation, the expected number of HBM accesses is
$O\!\left(\varepsilon^{-2}\sigma^{2}\log n \cdot nd\!\left(1 + d\sigma/M\right)\right)$,
which is near-linear in $n$.

\subsection{Proof of \Cref{prop:hbm}}
Recall that the algorithm of \Cref{thm:alg} draws $s$ i.i.d.\ samples $M_1,\ldots,M_s \sim \mathcal{M}$. For each sample $M_r$, if its blocks have sizes $|I_1|,\ldots,|I_t|$ (with $\sum_j |I_j|=n$), the attention computation restricted to $M_r$ reduces to $t$ independent local attention operations over sequences of lengths $|I_1|,\ldots,|I_t|$. We first explain how the outputs of these local operations are combined, and then bound the number of HBM accesses.

\textbf{Combining the local attention outputs.}
Since FlashAttention returns normalized attention outputs, whereas the algorithm of \Cref{thm:alg} averages the sampled masks before normalization, the local outputs must be combined with appropriate weights. Let $A'=A$ in the non-causal case and $A'=A\odot J$ in the causal case. For a row $i$ and a sample $r\in\{1,\ldots,s\}$, let
\[
  z_{ir}=\sum_{k=1}^nA'_{ik}(M_r)_{ik} \quad \text{and} \quad o_{ir}=\frac{1}{z_{ir}}\sum_{k=1}^nA'_{ik}(M_r)_{ik}V_{k,*} .
\]
Since the blocks of $M_r$ are all ones and partition $\{1,\ldots,n\}$, $z_{ir}$ and $o_{ir}$ are the softmax normalizer and the output of row $i$ in the local attention operation over the block of $M_r$ that contains $i$. Since $\widetilde M=\frac1s\sum_{r=1}^sM_r$, the output of the algorithm is
\begin{equation}\label{eq:combine}
  \widetilde T^\star_{i,*}
  = \frac{\sum_{k=1}^nA'_{ik}\widetilde M_{ik}V_{k,*}}{\sum_{k=1}^nA'_{ik}\widetilde M_{ik}}
  = \frac{\sum_{r=1}^s z_{ir}\,o_{ir}}{\sum_{r=1}^s z_{ir}},
\end{equation}
where the factors $\frac1s$ cancel. FlashAttention maintains for each row the maximum $m$ of the attention scores and the sum $\ell$ of the exponentiated scores shifted by $m$, and writes both to HBM (see Algorithm~1 in \cite{dao2022flashattention}); the normalizer is then $z=e^m\ell$. Hence each local operation provides both $o_{ir}$ and $z_{ir}$, and the outputs can be combined by \cref{eq:combine} (in practice, in a numerically stable way using $\log z_{ir}=m+\log\ell$).

\begin{proof}[Proof of \Cref{prop:hbm}]
Fix a single sample $M_r \sim \mathcal{M}$, with contiguous blocks $I_1,\ldots,I_t$ partitioning
$\{1,\ldots,n\}$, induced by a draw $b \sim \Gamma(2,\sigma)$ from the RBF scheme of
\Cref{def:rbf}. By that construction, every block has size $|I_j| \leq b + 1$.

\textbf{Local attention operations.}
We apply Algorithm~1 of \cite{dao2022flashattention} as a black box to each block independently.
On a sequence of length $N$, it uses column blocks of size $B_c=\left\lceil \frac{M}{4d}\right\rceil$, and by the
proof of Theorem~2 in \cite{dao2022flashattention}, it incurs $\Theta(NdT_c)$ HBM accesses, where
$T_c=\left\lceil \frac{N}{B_c}\right\rceil$ is the number of column blocks. Since $B_c\geq \frac{M}{4d}$, we have
$T_c\leq 1+\frac{4Nd}{M}$. Hence, for a block of size $|I_j|$, the number of HBM accesses is
$O\!\left(|I_j|\,d\left(1+\frac{|I_j|\,d}{M}\right)\right)$.
Summing over all blocks of a single sample:
\[
    \text{HBM accesses for one sample}
    \;\leq\; O\!\left(\sum_{j=1}^{t} |I_j|\,d\right)
           + O\!\left(\sum_{j=1}^{t} \frac{|I_j|^2\,d^2}{M}\right).
\]
The first sum is $O(nd)$. For the second sum, taking expectation
over $b \sim \Gamma(2,\sigma)$ and using $\mathbb{E}[b] = 2\sigma$ from \Cref{fct:gamma}:
\[
    \mathbb{E}\!\left[\sum_{j=1}^{t} \frac{|I_j|^2\, d^2}{M}\right]
    \;\leq\; \frac{d^2}{M}\,\mathbb{E}\!\left[\max_j|I_j| \cdot \sum_{j=1}^{t}|I_j|\right]
    \;\leq\; \frac{d^2}{M} \cdot \mathbb{E}[b+1] \cdot n
    \;=\; O\!\left(\frac{n d^2 \sigma}{M}\right).
\]

\textbf{Combining the samples.}
The outputs $o_{ir}$ and the values $m$, $\ell$ are written to HBM by FlashAttention itself, so the cost of these writes is already included in the bound for the local attention operations. Computing the rows $\widetilde T^\star_{i,*}$ for all $i$ by \cref{eq:combine} then requires reading them once for all $s$ samples, which takes $O(snd)$ HBM accesses.

Summing the cost of the local operations over all $s$ samples and adding the cost of combining them, the total expected number of HBM accesses is
$O\!\left(snd\!\left(1 + \frac{d\sigma}{M}\right)\right)$.
\end{proof}

\section{Generality: Positional LSH beyond ALiBi}\label{sec:generality}
So far we have focused on ALiBi as a concrete case of a popular positional bias that adheres to our positional LSH framework. In this section, we draw a distinction between the results that extend to any positional bias scheme that adheres to \Cref{sec:plsh_generic} and the results specific to ALiBi. 

\emph{What extends to any positional LSH:} Let $L$ be any positional bias matrix that satisfies \cref{eq:lsh} in \Cref{sec:plsh_generic}. Then the identity $\E_{M\sim\mathcal M}[M]=L$ in \Cref{thm:main} immediately holds. Furthermore, the max-norm convergence bound, \cref{eq:maxnorm_concentration} in \Cref{thm:main}, also holds, since its proof relies only on scalar Bernoulli concentration (\Cref{lmm:maxnorm_app}). Consequently, the uniform attention approximation guarantee of \Cref{thm:alg} holds in a slightly weaker form, with $\Delta_P'=C\sqrt{\log(n/\delta)/s}$ instead of $\Delta_P$.

\emph{What is specific to ALiBi:} 
In \Cref{thm:main}, the spectral norm convergence bounds  \cref{eq:spectral_expectation,eq:spectral_concentration} in \Cref{thm:main} and the block size bound \cref{eq:thm_blocksize} rely on the specific properties of the RBFs LSH scheme for ALiBi (\Cref{def:rbf}), particularly on properties of its underlying Gamma distribution and its interplay with the space of indices in a sequence (\Cref{sec:matrixnorms,sec:blocksize}). This also includes the contiguity of the blocks (see \Cref{sec:plsh_generic,sec:rbf}). For other positional bias and LSH schemes, the analogous results would require their own proofs. Note that the bound on the block size is key to the running time of the positionally-biased attention approximation algorithm in \Cref{thm:alg}.

As a concrete positive example, consider the logarithmic KERPLE positional bias \cite{chi2022kerple}. Its additive logit bias for positions $i,j$ is $\mathrm{bias}(i,j)=-r_1\log(1+r_2|i-j|)$ with parameters $r_1,r_2>0$. In our notation, the corresponding multiplicative bias matrix is $L_{ij}^{\mathrm{KERPLE}}=e^{\mathrm{bias}(i,j)}=(1+r_2|i-j|)^{-r_1}$. By the standard formula for the Laplace transform of the Gamma distribution, one can check that $L_{ij}^{\mathrm{KERPLE}}=\E_{\lambda\sim\Gamma(r_1,r_2)}[e^{-\lambda|i-j|}]$. Thus, $L_{ij}^{\mathrm{KERPLE}}$ is a Gamma-weighted mixture of Laplace kernels. This gives rise to a positional LSH scheme for KERPLE: sample $\lambda\sim\Gamma(r_1,r_2)$ and then run RBF with bandwidth $1/\lambda$ (i.e., conditioned on $\lambda$, the RBF bin width is sampled from $\Gamma(2,1/\lambda)$). 

On the other hand, one can  establish negative examples, where no (exact) positional LSH scheme exists. 
For example, it is known that the Gaussian kernel $\ker(u,u')=e^{-\norm{u-u'}_2^2}$ over $\R^d$ has no LSH scheme even for $d=1$ \cite{gorodissky2026locality}. In particular, no LSH scheme exists for $e^{-\frac14|i-j|^2}$ over integers $i,j$. 
To see this, in \cite{charikar2002similarity} it is proven that if $\kk(u,u')$ admits an LSH scheme then $1-\kk(u,u')$ must satisfy the triangle inequality over $U$; however, one can check the triangle inequality $1-e^{-\frac14|i-k|^2}\leq 1-e^{-\frac14|i-j|^2}+1-e^{-\frac14|j-k|^2}$ fails for any triplet of consecutive integers $i,j,k$. 
Therefore, an additive quadratic bias term $-\frac14|i-j|^2$ cannot have an exact positional LSH scheme.


\section{Experimental Validation}

\subsection{Convergence in Matrix Norms}
\label{sec:theoretical-validation}

First, we validate our theoretical results on convergence in matrix norms in \Cref{thm:main,thm:alg} directly on attention operations in pre-trained models. We use two publicly available large language models: 
Qwen3-14B and Mistral-7B. 
We use as inputs 30 text segments from the Wikitext-103 dataset \cite{merity2016pointer} with context lengths of 4k-4.5k and record the output of the models.
We then retrieve the corresponding $Q$,$K$,$V$ matrices for each model per input, layer and attention head.
We compute $\|L^{\star} - \widetilde{M}\|$, $\|L^{\star} - \widetilde{M}\|_{\max}$ and  $\norm{T^\star - \widetilde T^\star}_{\max}$, for varying sample sizes. 

The results are reported in Figure~\ref{fig:theoretical-validation-heads}.
They show that for both models, norm errors decay towards zero as the sample size increases. 

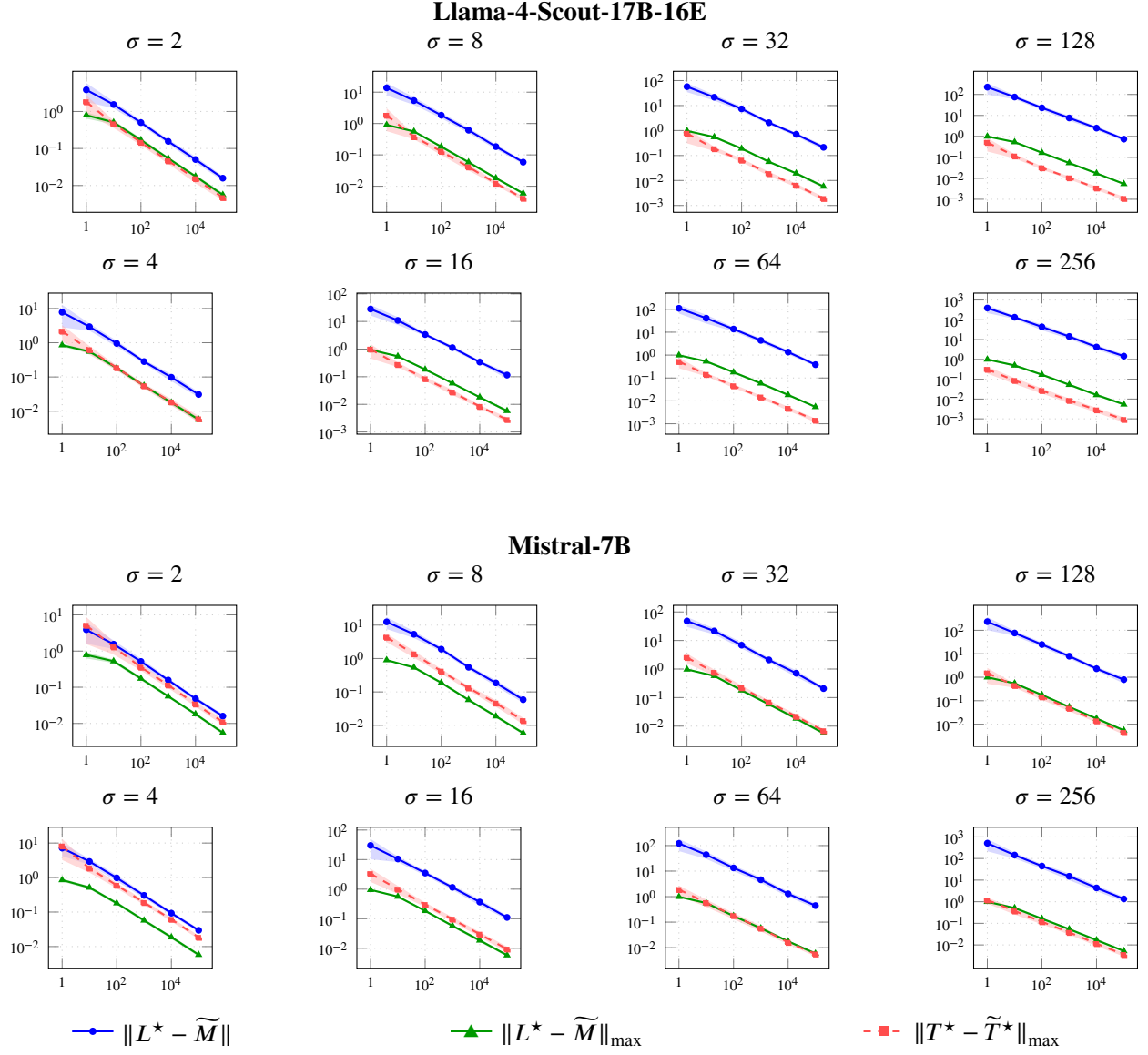
\begin{figure}
\centering
 
\par\noindent\centering\small\paragraph{Qwen3-14B\\
}
\par\smallskip

\begin{tikzpicture}
\begin{axis}[
    width=0.24\textwidth, height=0.22\textwidth,
    title={$\sigma = 2$},
    grid=major, grid style={dotted, gray!40},
    xtick={0,2,4}, xticklabels={$1$, $10^2$, $10^4$},
    ytick={1,0,-1,-2,-3}, yticklabels={$10^1$, $10^0$, $10^{-1}$, $10^{-2}$, $10^{-3}$},
    tick label style={font=\tiny}, label style={font=\tiny},
]
\addplot[blue, thick, mark=*, mark options={fill=blue}, mark size=1pt]
    coordinates {(0,0.548)(1,0.191)(2,-0.306)(3,-0.792)(4,-1.307)(5,-1.807)};
\addplot[blue!30, draw=none, name path=normHi_q2, forget plot]
    coordinates {(0,0.715)(1,0.276)(2,-0.245)(3,-0.742)(4,-1.253)(5,-1.737)};
\addplot[blue!30, draw=none, name path=normLo_q2, forget plot]
    coordinates {(0,0.273)(1,0.085)(2,-0.378)(3,-0.849)(4,-1.369)(5,-1.889)};
\addplot[blue!30, opacity=0.5, forget plot] fill between[of=normHi_q2 and normLo_q2];
\addplot[green!60!black, thick, mark=triangle*, mark options={fill=green!60!black}, mark size=1pt]
    coordinates {(0,-0.115)(1,-0.270)(2,-0.744)(3,-1.241)(4,-1.752)(5,-2.245)};
\addplot[green!30, draw=none, name path=mnHi_q2, forget plot]
    coordinates {(0,-0.035)(1,-0.228)(2,-0.697)(3,-1.193)(4,-1.706)(5,-2.197)};
\addplot[green!30, draw=none, name path=mnLo_q2, forget plot]
    coordinates {(0,-0.214)(1,-0.316)(2,-0.797)(3,-1.296)(4,-1.803)(5,-2.299)};
\addplot[green!60!black, opacity=0.2, forget plot] fill between[of=mnHi_q2 and mnLo_q2];
\addplot[red!70, thick, mark=square*, mark options={fill=red!70}, mark size=1pt, dashed]
    coordinates {(0,1.108)(1,0.542)(2,0.036)(3,-0.465)(4,-0.956)(5,-1.453)};
\addplot[red!30, draw=none, name path=softHi_q2, forget plot]
    coordinates {(0,1.271)(1,0.623)(2,0.126)(3,-0.397)(4,-0.877)(5,-1.362)};
\addplot[red!30, draw=none, name path=softLo_q2, forget plot]
    coordinates {(0,0.843)(1,0.443)(2,-0.079)(3,-0.546)(4,-1.052)(5,-1.569)};
\addplot[red!30, opacity=0.5, forget plot] fill between[of=softHi_q2 and softLo_q2];
\end{axis}
\end{tikzpicture}
\hfill
\begin{tikzpicture}
\begin{axis}[
    width=0.24\textwidth, height=0.22\textwidth,
    title={$\sigma = 8$},
    grid=major, grid style={dotted, gray!40},
    xtick={0,2,4}, xticklabels={$1$, $10^2$, $10^4$},
    ytick={1,0,-1,-2,-3}, yticklabels={$10^1$, $10^0$, $10^{-1}$, $10^{-2}$, $10^{-3}$},
    tick label style={font=\tiny}, label style={font=\tiny},
]
\addplot[blue, thick, mark=*, mark options={fill=blue}, mark size=1pt]
    coordinates {(0,1.130)(1,0.744)(2,0.259)(3,-0.226)(4,-0.731)(5,-1.238)};
\addplot[blue!30, draw=none, name path=normHi_q8, forget plot]
    coordinates {(0,1.349)(1,0.839)(2,0.336)(3,-0.157)(4,-0.647)(5,-1.176)};
\addplot[blue!30, draw=none, name path=normLo_q8, forget plot]
    coordinates {(0,0.667)(1,0.622)(2,0.165)(3,-0.308)(4,-0.837)(5,-1.311)};
\addplot[blue!30, opacity=0.5, forget plot] fill between[of=normHi_q8 and normLo_q8];
\addplot[green!60!black, thick, mark=triangle*, mark options={fill=green!60!black}, mark size=1pt]
    coordinates {(0,-0.046)(1,-0.256)(2,-0.740)(3,-1.230)(4,-1.744)(5,-2.235)};
\addplot[green!30, draw=none, name path=mnHi_q8, forget plot]
    coordinates {(0,-0.029)(1,-0.216)(2,-0.701)(3,-1.181)(4,-1.698)(5,-2.191)};
\addplot[green!30, draw=none, name path=mnLo_q8, forget plot]
    coordinates {(0,-0.064)(1,-0.299)(2,-0.783)(3,-1.286)(4,-1.795)(5,-2.285)};
\addplot[green!60!black, opacity=0.2, forget plot] fill between[of=mnHi_q8 and mnLo_q8];
\addplot[red!70, thick, mark=square*, mark options={fill=red!70}, mark size=1pt, dashed]
    coordinates {(0,0.746)(1,0.124)(2,-0.367)(3,-0.852)(4,-1.379)(5,-1.879)};
\addplot[red!30, draw=none, name path=softHi_q8, forget plot]
    coordinates {(0,0.930)(1,0.202)(2,-0.291)(3,-0.773)(4,-1.308)(5,-1.817)};
\addplot[red!30, draw=none, name path=softLo_q8, forget plot]
    coordinates {(0,0.424)(1,0.029)(2,-0.459)(3,-0.949)(4,-1.465)(5,-1.950)};
\addplot[red!30, opacity=0.5, forget plot] fill between[of=softHi_q8 and softLo_q8];
\end{axis}
\end{tikzpicture}
\hfill
\begin{tikzpicture}
\begin{axis}[
    width=0.24\textwidth, height=0.22\textwidth,
    title={$\sigma = 32$},
    grid=major, grid style={dotted, gray!40},
    xtick={0,2,4}, xticklabels={$1$, $10^2$, $10^4$},
    ytick={2,1,0,-1,-2,-3}, yticklabels={$10^2$, $10^1$, $10^0$, $10^{-1}$, $10^{-2}$, $10^{-3}$},
    tick label style={font=\tiny}, label style={font=\tiny},
]
\addplot[blue, thick, mark=*, mark options={fill=blue}, mark size=1pt]
    coordinates {(0,1.744)(1,1.351)(2,0.863)(3,0.315)(4,-0.144)(5,-0.661)};
\addplot[blue!30, draw=none, name path=normHi_q32, forget plot]
    coordinates {(0,2.031)(1,1.452)(2,0.949)(3,0.395)(4,-0.041)(5,-0.559)};
\addplot[blue!30, draw=none, name path=normLo_q32, forget plot]
    coordinates {(0,0.533)(1,1.220)(2,0.754)(3,0.218)(4,-0.281)(5,-0.795)};
\addplot[blue!30, opacity=0.5, forget plot] fill between[of=normHi_q32 and normLo_q32];
\addplot[green!60!black, thick, mark=triangle*, mark options={fill=green!60!black}, mark size=1pt]
    coordinates {(0,-0.013)(1,-0.249)(2,-0.719)(3,-1.254)(4,-1.727)(5,-2.227)};
\addplot[green!30, draw=none, name path=mnHi_q32, forget plot]
    coordinates {(0,-0.010)(1,-0.207)(2,-0.664)(3,-1.205)(4,-1.686)(5,-2.172)};
\addplot[green!30, draw=none, name path=mnLo_q32, forget plot]
    coordinates {(0,-0.016)(1,-0.296)(2,-0.782)(3,-1.309)(4,-1.773)(5,-2.289)};
\addplot[green!60!black, opacity=0.2, forget plot] fill between[of=mnHi_q32 and mnLo_q32];
\addplot[red!70, thick, mark=square*, mark options={fill=red!70}, mark size=1pt, dashed]
    coordinates {(0,0.470)(1,-0.140)(2,-0.642)(3,-1.159)(4,-1.636)(5,-2.149)};
\addplot[red!30, draw=none, name path=softHi_q32, forget plot]
    coordinates {(0,0.638)(1,-0.053)(2,-0.579)(3,-1.106)(4,-1.549)(5,-2.081)};
\addplot[red!30, draw=none, name path=softLo_q32, forget plot]
    coordinates {(0,0.193)(1,-0.250)(2,-0.717)(3,-1.219)(4,-1.745)(5,-2.231)};
\addplot[red!30, opacity=0.5, forget plot] fill between[of=softHi_q32 and softLo_q32];
\end{axis}
\end{tikzpicture}
\hfill
\begin{tikzpicture}
\begin{axis}[
    width=0.24\textwidth, height=0.22\textwidth,
    title={$\sigma = 128$},
    grid=major, grid style={dotted, gray!40},
    xtick={0,2,4}, xticklabels={$1$, $10^2$, $10^4$},
    ytick={2,1,0,-1,-2,-3}, yticklabels={$10^2$, $10^1$, $10^0$, $10^{-1}$, $10^{-2}$, $10^{-3}$},
    tick label style={font=\tiny}, label style={font=\tiny},
]
\addplot[blue, thick, mark=*, mark options={fill=blue}, mark size=1pt]
    coordinates {(0,2.365)(1,1.876)(2,1.392)(3,0.893)(4,0.390)(5,-0.147)};
\addplot[blue!30, draw=none, name path=normHi_q128, forget plot]
    coordinates {(0,2.526)(1,1.975)(2,1.465)(3,0.992)(4,0.488)(5,-0.038)};
\addplot[blue!30, draw=none, name path=normLo_q128, forget plot]
    coordinates {(0,2.106)(1,1.747)(2,1.304)(3,0.763)(4,0.265)(5,-0.293)};
\addplot[blue!30, opacity=0.5, forget plot] fill between[of=normHi_q128 and normLo_q128];
\addplot[green!60!black, thick, mark=triangle*, mark options={fill=green!60!black}, mark size=1pt]
    coordinates {(0,-0.003)(1,-0.277)(2,-0.762)(3,-1.261)(4,-1.763)(5,-2.265)};
\addplot[green!30, draw=none, name path=mnHi_q128, forget plot]
    coordinates {(0,-0.003)(1,-0.237)(2,-0.714)(3,-1.214)(4,-1.709)(5,-2.218)};
\addplot[green!30, draw=none, name path=mnLo_q128, forget plot]
    coordinates {(0,-0.004)(1,-0.322)(2,-0.817)(3,-1.315)(4,-1.825)(5,-2.318)};
\addplot[green!60!black, opacity=0.2, forget plot] fill between[of=mnHi_q128 and mnLo_q128];
\addplot[red!70, thick, mark=square*, mark options={fill=red!70}, mark size=1pt, dashed]
    coordinates {(0,0.145)(1,-0.362)(2,-0.891)(3,-1.364)(4,-1.883)(5,-2.407)};
\addplot[red!30, draw=none, name path=softHi_q128, forget plot]
    coordinates {(0,0.337)(1,-0.263)(2,-0.826)(3,-1.291)(4,-1.814)(5,-2.325)};
\addplot[red!30, draw=none, name path=softLo_q128, forget plot]
    coordinates {(0,-0.207)(1,-0.490)(2,-0.967)(3,-1.451)(4,-1.964)(5,-2.507)};
\addplot[red!30, opacity=0.5, forget plot] fill between[of=softHi_q128 and softLo_q128];
\end{axis}
\end{tikzpicture}

\begin{tikzpicture}
\begin{axis}[
    width=0.24\textwidth, height=0.22\textwidth,
    title={$\sigma = 4$},
    grid=major, grid style={dotted, gray!40},
    xtick={0,2,4}, xticklabels={$1$, $10^2$, $10^4$},
    ytick={1,0,-1,-2,-3}, yticklabels={$10^1$, $10^0$, $10^{-1}$, $10^{-2}$, $10^{-3}$},
    tick label style={font=\tiny}, label style={font=\tiny},
]
\addplot[blue, thick, mark=*, mark options={fill=blue}, mark size=1pt]
    coordinates {(0,0.793)(1,0.492)(2,-0.010)(3,-0.516)(4,-1.033)(5,-1.510)};
\addplot[blue!30, draw=none, name path=normHi_q4, forget plot]
    coordinates {(0,0.932)(1,0.585)(2,0.062)(3,-0.456)(4,-0.984)(5,-1.416)};
\addplot[blue!30, draw=none, name path=normLo_q4, forget plot]
    coordinates {(0,0.587)(1,0.373)(2,-0.098)(3,-0.586)(4,-1.089)(5,-1.630)};
\addplot[blue!30, opacity=0.5, forget plot] fill between[of=normHi_q4 and normLo_q4];
\addplot[green!60!black, thick, mark=triangle*, mark options={fill=green!60!black}, mark size=1pt]
    coordinates {(0,-0.083)(1,-0.265)(2,-0.738)(3,-1.246)(4,-1.752)(5,-2.239)};
\addplot[green!30, draw=none, name path=mnHi_q4, forget plot]
    coordinates {(0,-0.047)(1,-0.224)(2,-0.701)(3,-1.192)(4,-1.713)(5,-2.192)};
\addplot[green!30, draw=none, name path=mnLo_q4, forget plot]
    coordinates {(0,-0.123)(1,-0.309)(2,-0.778)(3,-1.308)(4,-1.796)(5,-2.292)};
\addplot[green!60!black, opacity=0.2, forget plot] fill between[of=mnHi_q4 and mnLo_q4];
\addplot[red!70, thick, mark=square*, mark options={fill=red!70}, mark size=1pt, dashed]
    coordinates {(0,0.951)(1,0.275)(2,-0.211)(3,-0.721)(4,-1.249)(5,-1.707)};
\addplot[red!30, draw=none, name path=softHi_q4, forget plot]
    coordinates {(0,1.144)(1,0.353)(2,-0.139)(3,-0.657)(4,-1.210)(5,-1.624)};
\addplot[red!30, draw=none, name path=softLo_q4, forget plot]
    coordinates {(0,0.596)(1,0.181)(2,-0.297)(3,-0.797)(4,-1.292)(5,-1.809)};
\addplot[red!30, opacity=0.5, forget plot] fill between[of=softHi_q4 and softLo_q4];
\end{axis}
\end{tikzpicture}
\hfill
\begin{tikzpicture}
\begin{axis}[
    width=0.24\textwidth, height=0.22\textwidth,
    title={$\sigma = 16$},
    grid=major, grid style={dotted, gray!40},
    xtick={0,2,4}, xticklabels={$1$, $10^2$, $10^4$},
    ytick={2,1,0,-1,-2,-3}, yticklabels={$10^2$, $10^1$, $10^0$, $10^{-1}$, $10^{-2}$, $10^{-3}$},
    tick label style={font=\tiny}, label style={font=\tiny},
]
\addplot[blue, thick, mark=*, mark options={fill=blue}, mark size=1pt]
    coordinates {(0,1.463)(1,1.064)(2,0.542)(3,0.063)(4,-0.458)(5,-0.945)};
\addplot[blue!30, draw=none, name path=normHi_q16, forget plot]
    coordinates {(0,1.624)(1,1.188)(2,0.625)(3,0.145)(4,-0.376)(5,-0.867)};
\addplot[blue!30, draw=none, name path=normLo_q16, forget plot]
    coordinates {(0,1.203)(1,0.888)(2,0.440)(3,-0.038)(4,-0.560)(5,-1.041)};
\addplot[blue!30, opacity=0.5, forget plot] fill between[of=normHi_q16 and normLo_q16];
\addplot[green!60!black, thick, mark=triangle*, mark options={fill=green!60!black}, mark size=1pt]
    coordinates {(0,-0.024)(1,-0.264)(2,-0.729)(3,-1.230)(4,-1.740)(5,-2.228)};
\addplot[green!30, draw=none, name path=mnHi_q16, forget plot]
    coordinates {(0,-0.017)(1,-0.230)(2,-0.675)(3,-1.182)(4,-1.678)(5,-2.181)};
\addplot[green!30, draw=none, name path=mnLo_q16, forget plot]
    coordinates {(0,-0.032)(1,-0.301)(2,-0.791)(3,-1.284)(4,-1.812)(5,-2.280)};
\addplot[green!60!black, opacity=0.2, forget plot] fill between[of=mnHi_q16 and mnLo_q16];
\addplot[red!70, thick, mark=square*, mark options={fill=red!70}, mark size=1pt, dashed]
    coordinates {(0,0.677)(1,-0.017)(2,-0.513)(3,-0.997)(4,-1.503)(5,-2.018)};
\addplot[red!30, draw=none, name path=softHi_q16, forget plot]
    coordinates {(0,0.914)(1,0.081)(2,-0.448)(3,-0.920)(4,-1.436)(5,-1.964)};
\addplot[red!30, draw=none, name path=softLo_q16, forget plot]
    coordinates {(0,0.111)(1,-0.145)(2,-0.589)(3,-1.090)(4,-1.583)(5,-2.080)};
\addplot[red!30, opacity=0.5, forget plot] fill between[of=softHi_q16 and softLo_q16];
\end{axis}
\end{tikzpicture}
\hfill
\begin{tikzpicture}
\begin{axis}[
    width=0.24\textwidth, height=0.22\textwidth,
    title={$\sigma = 64$},
    grid=major, grid style={dotted, gray!40},
    xtick={0,2,4}, xticklabels={$1$, $10^2$, $10^4$},
    ytick={2,1,0,-1,-2,-3}, yticklabels={$10^2$, $10^1$, $10^0$, $10^{-1}$, $10^{-2}$, $10^{-3}$},
    tick label style={font=\tiny}, label style={font=\tiny},
]
\addplot[blue, thick, mark=*, mark options={fill=blue}, mark size=1pt]
    coordinates {(0,2.071)(1,1.619)(2,1.126)(3,0.619)(4,0.130)(5,-0.408)};
\addplot[blue!30, draw=none, name path=normHi_q64, forget plot]
    coordinates {(0,2.254)(1,1.728)(2,1.217)(3,0.702)(4,0.235)(5,-0.321)};
\addplot[blue!30, draw=none, name path=normLo_q64, forget plot]
    coordinates {(0,1.748)(1,1.472)(2,1.012)(3,0.517)(4,-0.007)(5,-0.516)};
\addplot[blue!30, opacity=0.5, forget plot] fill between[of=normHi_q64 and normLo_q64];
\addplot[green!60!black, thick, mark=triangle*, mark options={fill=green!60!black}, mark size=1pt]
    coordinates {(0,-0.006)(1,-0.250)(2,-0.747)(3,-1.261)(4,-1.716)(5,-2.253)};
\addplot[green!30, draw=none, name path=mnHi_q64, forget plot]
    coordinates {(0,-0.005)(1,-0.201)(2,-0.700)(3,-1.217)(4,-1.657)(5,-2.209)};
\addplot[green!30, draw=none, name path=mnLo_q64, forget plot]
    coordinates {(0,-0.008)(1,-0.306)(2,-0.799)(3,-1.310)(4,-1.785)(5,-2.301)};
\addplot[green!60!black, opacity=0.2, forget plot] fill between[of=mnHi_q64 and mnLo_q64];
\addplot[red!70, thick, mark=square*, mark options={fill=red!70}, mark size=1pt, dashed]
    coordinates {(0,0.291)(1,-0.248)(2,-0.765)(3,-1.248)(4,-1.773)(5,-2.274)};
\addplot[red!30, draw=none, name path=softHi_q64, forget plot]
    coordinates {(0,0.454)(1,-0.102)(2,-0.678)(3,-1.178)(4,-1.700)(5,-2.219)};
\addplot[red!30, draw=none, name path=softLo_q64, forget plot]
    coordinates {(0,0.027)(1,-0.469)(2,-0.874)(3,-1.331)(4,-1.860)(5,-2.338)};
\addplot[red!30, opacity=0.5, forget plot] fill between[of=softHi_q64 and softLo_q64];
\end{axis}
\end{tikzpicture}
\hfill
\begin{tikzpicture}
\begin{axis}[
    width=0.24\textwidth, height=0.22\textwidth,
    title={$\sigma = 256$},
    grid=major, grid style={dotted, gray!40},
    xtick={0,2,4}, xticklabels={$1$, $10^2$, $10^4$},
    ytick={3,2,1,0,-1,-2,-3}, yticklabels={$10^3$, $10^2$, $10^1$, $10^0$, $10^{-1}$, $10^{-2}$, $10^{-3}$},
    tick label style={font=\tiny}, label style={font=\tiny},
]
\addplot[blue, thick, mark=*, mark options={fill=blue}, mark size=1pt]
    coordinates {(0,2.626)(1,2.195)(2,1.637)(3,1.152)(4,0.675)(5,0.173)};
\addplot[blue!30, draw=none, name path=normHi_q256, forget plot]
    coordinates {(0,2.773)(1,2.301)(2,1.729)(3,1.271)(4,0.769)(5,0.267)};
\addplot[blue!30, draw=none, name path=normLo_q256, forget plot]
    coordinates {(0,2.401)(1,2.054)(2,1.519)(3,0.988)(4,0.556)(5,0.054)};
\addplot[blue!30, opacity=0.5, forget plot] fill between[of=normHi_q256 and normLo_q256];
\addplot[green!60!black, thick, mark=triangle*, mark options={fill=green!60!black}, mark size=1pt]
    coordinates {(0,-0.002)(1,-0.304)(2,-0.790)(3,-1.270)(4,-1.785)(5,-2.268)};
\addplot[green!30, draw=none, name path=mnHi_q256, forget plot]
    coordinates {(0,-0.002)(1,-0.251)(2,-0.743)(3,-1.206)(4,-1.736)(5,-2.230)};
\addplot[green!30, draw=none, name path=mnLo_q256, forget plot]
    coordinates {(0,-0.002)(1,-0.366)(2,-0.843)(3,-1.346)(4,-1.841)(5,-2.310)};
\addplot[green!60!black, opacity=0.2, forget plot] fill between[of=mnHi_q256 and mnLo_q256];
\addplot[red!70, thick, mark=square*, mark options={fill=red!70}, mark size=1pt, dashed]
    coordinates {(0,0.152)(1,-0.460)(2,-0.987)(3,-1.476)(4,-1.988)(5,-2.475)};
\addplot[red!30, draw=none, name path=softHi_q256, forget plot]
    coordinates {(0,0.381)(1,-0.364)(2,-0.904)(3,-1.364)(4,-1.895)(5,-2.400)};
\addplot[red!30, draw=none, name path=softLo_q256, forget plot]
    coordinates {(0,-0.365)(1,-0.584)(2,-1.088)(3,-1.627)(4,-2.107)(5,-2.565)};
\addplot[red!30, opacity=0.5, forget plot] fill between[of=softHi_q256 and softLo_q256];
\end{axis}
\end{tikzpicture}

\par\medskip
\par\noindent\centering\small\paragraph{Mistral-7B\\
}
\par\smallskip
 
\begin{tikzpicture}
\begin{axis}[
    width=0.24\textwidth, height=0.22\textwidth,
    title={$\sigma = 2$},
    grid=major, grid style={dotted, gray!40},
    xtick={0,2,4}, xticklabels={$1$, $10^2$, $10^4$},
    ytick={1,0,-1,-2}, yticklabels={$10^1$, $10^0$, $10^{-1}$, $10^{-2}$},
    tick label style={font=\tiny}, label style={font=\tiny},
]
\addplot[blue, thick, mark=*, mark options={fill=blue}, mark size=1pt]
    coordinates {(0,0.590)(1,0.190)(2,-0.288)(3,-0.802)(4,-1.317)(5,-1.798)};
\addplot[blue!30, draw=none, name path=normHi_m2, forget plot]
    coordinates {(0,0.791)(1,0.292)(2,-0.206)(3,-0.750)(4,-1.265)(5,-1.744)};
\addplot[blue!30, draw=none, name path=normLo_m2, forget plot]
    coordinates {(0,0.204)(1,0.058)(2,-0.389)(3,-0.862)(4,-1.376)(5,-1.859)};
\addplot[blue!30, opacity=0.5, forget plot] fill between[of=normHi_m2 and normLo_m2];
\addplot[green!60!black, thick, mark=triangle*, mark options={fill=green!60!black}, mark size=1pt]
    coordinates {(0,-0.107)(1,-0.280)(2,-0.760)(3,-1.250)(4,-1.743)(5,-2.259)};
\addplot[green!30, draw=none, name path=mnHi_m2, forget plot]
    coordinates {(0,-0.030)(1,-0.236)(2,-0.714)(3,-1.213)(4,-1.703)(5,-2.213)};
\addplot[green!30, draw=none, name path=mnLo_m2, forget plot]
    coordinates {(0,-0.203)(1,-0.330)(2,-0.812)(3,-1.291)(4,-1.787)(5,-2.310)};
\addplot[green!60!black, opacity=0.2, forget plot] fill between[of=mnHi_m2 and mnLo_m2];
\addplot[red!70, thick, mark=square*, mark options={fill=red!70}, mark size=1pt, dashed]
    coordinates {(0,0.710)(1,0.114)(2,-0.447)(3,-0.940)(4,-1.465)(5,-1.958)};
\addplot[red!30, draw=none, name path=softHi_m2, forget plot]
    coordinates {(0,0.939)(1,0.254)(2,-0.372)(3,-0.870)(4,-1.404)(5,-1.884)};
\addplot[red!30, draw=none, name path=softLo_m2, forget plot]
    coordinates {(0,0.191)(1,-0.095)(2,-0.538)(3,-1.024)(4,-1.536)(5,-2.047)};
\addplot[red!30, opacity=0.5, forget plot] fill between[of=softHi_m2 and softLo_m2];
\end{axis}
\end{tikzpicture}
\hfill
\begin{tikzpicture}
\begin{axis}[
    width=0.24\textwidth, height=0.22\textwidth,
    title={$\sigma = 8$},
    grid=major, grid style={dotted, gray!40},
    xtick={0,2,4}, xticklabels={$1$, $10^2$, $10^4$},
    ytick={1,0,-1,-2}, yticklabels={$10^1$, $10^0$, $10^{-1}$, $10^{-2}$},
    tick label style={font=\tiny}, label style={font=\tiny},
]
\addplot[blue, thick, mark=*, mark options={fill=blue}, mark size=1pt]
    coordinates {(0,1.098)(1,0.727)(2,0.280)(3,-0.261)(4,-0.732)(5,-1.233)};
\addplot[blue!30, draw=none, name path=normHi_m8, forget plot]
    coordinates {(0,1.241)(1,0.817)(2,0.342)(3,-0.203)(4,-0.654)(5,-1.155)};
\addplot[blue!30, draw=none, name path=normLo_m8, forget plot]
    coordinates {(0,0.884)(1,0.613)(2,0.208)(3,-0.328)(4,-0.828)(5,-1.329)};
\addplot[blue!30, opacity=0.5, forget plot] fill between[of=normHi_m8 and normLo_m8];
\addplot[green!60!black, thick, mark=triangle*, mark options={fill=green!60!black}, mark size=1pt]
    coordinates {(0,-0.047)(1,-0.268)(2,-0.725)(3,-1.239)(4,-1.729)(5,-2.233)};
\addplot[green!30, draw=none, name path=mnHi_m8, forget plot]
    coordinates {(0,-0.032)(1,-0.230)(2,-0.687)(3,-1.199)(4,-1.687)(5,-2.190)};
\addplot[green!30, draw=none, name path=mnLo_m8, forget plot]
    coordinates {(0,-0.061)(1,-0.310)(2,-0.768)(3,-1.283)(4,-1.776)(5,-2.280)};
\addplot[green!60!black, opacity=0.2, forget plot] fill between[of=mnHi_m8 and mnLo_m8];
\addplot[red!70, thick, mark=square*, mark options={fill=red!70}, mark size=1pt, dashed]
    coordinates {(0,0.636)(1,0.137)(2,-0.378)(3,-0.881)(4,-1.335)(5,-1.862)};
\addplot[red!30, draw=none, name path=softHi_m8, forget plot]
    coordinates {(0,0.745)(1,0.226)(2,-0.312)(3,-0.817)(4,-1.248)(5,-1.776)};
\addplot[red!30, draw=none, name path=softLo_m8, forget plot]
    coordinates {(0,0.492)(1,0.024)(2,-0.456)(3,-0.955)(4,-1.443)(5,-1.969)};
\addplot[red!30, opacity=0.5, forget plot] fill between[of=softHi_m8 and softLo_m8];
\end{axis}
\end{tikzpicture}
\hfill
\begin{tikzpicture}
\begin{axis}[
    width=0.24\textwidth, height=0.22\textwidth,
    title={$\sigma = 32$},
    grid=major, grid style={dotted, gray!40},
    xtick={0,2,4}, xticklabels={$1$, $10^2$, $10^4$},
    ytick={2,1,0,-1,-2}, yticklabels={$10^2$, $10^1$, $10^0$, $10^{-1}$, $10^{-2}$},
    tick label style={font=\tiny}, label style={font=\tiny},
]
\addplot[blue, thick, mark=*, mark options={fill=blue}, mark size=1pt]
    coordinates {(0,1.680)(1,1.333)(2,0.836)(3,0.319)(4,-0.146)(5,-0.684)};
\addplot[blue!30, draw=none, name path=normHi_m32, forget plot]
    coordinates {(0,1.819)(1,1.426)(2,0.927)(3,0.391)(4,-0.045)(5,-0.627)};
\addplot[blue!30, draw=none, name path=normLo_m32, forget plot]
    coordinates {(0,1.474)(1,1.214)(2,0.721)(3,0.232)(4,-0.278)(5,-0.749)};
\addplot[blue!30, opacity=0.5, forget plot] fill between[of=normHi_m32 and normLo_m32];
\addplot[green!60!black, thick, mark=triangle*, mark options={fill=green!60!black}, mark size=1pt]
    coordinates {(0,-0.013)(1,-0.240)(2,-0.746)(3,-1.236)(4,-1.733)(5,-2.248)};
\addplot[green!30, draw=none, name path=mnHi_m32, forget plot]
    coordinates {(0,-0.012)(1,-0.202)(2,-0.697)(3,-1.201)(4,-1.682)(5,-2.204)};
\addplot[green!30, draw=none, name path=mnLo_m32, forget plot]
    coordinates {(0,-0.015)(1,-0.281)(2,-0.800)(3,-1.275)(4,-1.792)(5,-2.298)};
\addplot[green!60!black, opacity=0.2, forget plot] fill between[of=mnHi_m32 and mnLo_m32];
\addplot[red!70, thick, mark=square*, mark options={fill=red!70}, mark size=1pt, dashed]
    coordinates {(0,0.402)(1,-0.119)(2,-0.655)(3,-1.161)(4,-1.661)(5,-2.163)};
\addplot[red!30, draw=none, name path=softHi_m32, forget plot]
    coordinates {(0,0.539)(1,-0.020)(2,-0.575)(3,-1.103)(4,-1.584)(5,-2.101)};
\addplot[red!30, draw=none, name path=softLo_m32, forget plot]
    coordinates {(0,0.202)(1,-0.249)(2,-0.753)(3,-1.228)(4,-1.755)(5,-2.235)};
\addplot[red!30, opacity=0.5, forget plot] fill between[of=softHi_m32 and softLo_m32];
\end{axis}
\end{tikzpicture}
\hfill
\begin{tikzpicture}
\begin{axis}[
    width=0.24\textwidth, height=0.22\textwidth,
    title={$\sigma = 128$},
    grid=major, grid style={dotted, gray!40},
    xtick={0,2,4}, xticklabels={$1$, $10^2$, $10^4$},
    ytick={2,1,0,-1,-2}, yticklabels={$10^2$, $10^1$, $10^0$, $10^{-1}$, $10^{-2}$},
    tick label style={font=\tiny}, label style={font=\tiny},
]
\addplot[blue, thick, mark=*, mark options={fill=blue}, mark size=1pt]
    coordinates {(0,2.367)(1,1.883)(2,1.393)(3,0.896)(4,0.362)(5,-0.103)};
\addplot[blue!30, draw=none, name path=normHi_m128, forget plot]
    coordinates {(0,2.562)(1,1.976)(2,1.482)(3,0.977)(4,0.434)(5,0.014)};
\addplot[blue!30, draw=none, name path=normLo_m128, forget plot]
    coordinates {(0,2.006)(1,1.764)(2,1.281)(3,0.796)(4,0.277)(5,-0.263)};
\addplot[blue!30, opacity=0.5, forget plot] fill between[of=normHi_m128 and normLo_m128];
\addplot[green!60!black, thick, mark=triangle*, mark options={fill=green!60!black}, mark size=1pt]
    coordinates {(0,-0.003)(1,-0.262)(2,-0.752)(3,-1.266)(4,-1.760)(5,-2.263)};
\addplot[green!30, draw=none, name path=mnHi_m128, forget plot]
    coordinates {(0,-0.003)(1,-0.216)(2,-0.715)(3,-1.230)(4,-1.716)(5,-2.208)};
\addplot[green!30, draw=none, name path=mnLo_m128, forget plot]
    coordinates {(0,-0.004)(1,-0.313)(2,-0.792)(3,-1.304)(4,-1.810)(5,-2.325)};
\addplot[green!60!black, opacity=0.2, forget plot] fill between[of=mnHi_m128 and mnLo_m128];
\addplot[red!70, thick, mark=square*, mark options={fill=red!70}, mark size=1pt, dashed]
    coordinates {(0,0.180)(1,-0.354)(2,-0.837)(3,-1.325)(4,-1.862)(5,-2.356)};
\addplot[red!30, draw=none, name path=softHi_m128, forget plot]
    coordinates {(0,0.399)(1,-0.280)(2,-0.742)(3,-1.229)(4,-1.804)(5,-2.282)};
\addplot[red!30, draw=none, name path=softLo_m128, forget plot]
    coordinates {(0,-0.282)(1,-0.442)(2,-0.960)(3,-1.447)(4,-1.930)(5,-2.446)};
\addplot[red!30, opacity=0.5, forget plot] fill between[of=softHi_m128 and softLo_m128];
\end{axis}
\end{tikzpicture}

\begin{tikzpicture}
\begin{axis}[
    width=0.24\textwidth, height=0.22\textwidth,
    title={$\sigma = 4$},
    grid=major, grid style={dotted, gray!40},
    xtick={0,2,4}, xticklabels={$1$, $10^2$, $10^4$},
    ytick={1,0,-1,-2}, yticklabels={$10^1$, $10^0$, $10^{-1}$, $10^{-2}$},
    tick label style={font=\tiny}, label style={font=\tiny},
]
\addplot[blue, thick, mark=*, mark options={fill=blue}, mark size=1pt]
    coordinates {(0,0.854)(1,0.465)(2,-0.010)(3,-0.517)(4,-1.033)(5,-1.529)};
\addplot[blue!30, draw=none, name path=normHi_m4, forget plot]
    coordinates {(0,1.004)(1,0.541)(2,0.069)(3,-0.455)(4,-0.980)(5,-1.480)};
\addplot[blue!30, draw=none, name path=normLo_m4, forget plot]
    coordinates {(0,0.622)(1,0.372)(2,-0.106)(3,-0.590)(4,-1.094)(5,-1.584)};
\addplot[blue!30, opacity=0.5, forget plot] fill between[of=normHi_m4 and normLo_m4];
\addplot[green!60!black, thick, mark=triangle*, mark options={fill=green!60!black}, mark size=1pt]
    coordinates {(0,-0.072)(1,-0.290)(2,-0.743)(3,-1.243)(4,-1.730)(5,-2.238)};
\addplot[green!30, draw=none, name path=mnHi_m4, forget plot]
    coordinates {(0,-0.032)(1,-0.247)(2,-0.700)(3,-1.199)(4,-1.685)(5,-2.194)};
\addplot[green!30, draw=none, name path=mnLo_m4, forget plot]
    coordinates {(0,-0.117)(1,-0.337)(2,-0.791)(3,-1.292)(4,-1.780)(5,-2.288)};
\addplot[green!60!black, opacity=0.2, forget plot] fill between[of=mnHi_m4 and mnLo_m4];
\addplot[red!70, thick, mark=square*, mark options={fill=red!70}, mark size=1pt, dashed]
    coordinates {(0,0.913)(1,0.272)(2,-0.225)(3,-0.724)(4,-1.212)(5,-1.734)};
\addplot[red!30, draw=none, name path=softHi_m4, forget plot]
    coordinates {(0,1.119)(1,0.371)(2,-0.152)(3,-0.665)(4,-1.159)(5,-1.687)};
\addplot[red!30, draw=none, name path=softLo_m4, forget plot]
    coordinates {(0,0.510)(1,0.143)(2,-0.311)(3,-0.791)(4,-1.272)(5,-1.787)};
\addplot[red!30, opacity=0.5, forget plot] fill between[of=softHi_m4 and softLo_m4];
\end{axis}
\end{tikzpicture}
\hfill
\begin{tikzpicture}
\begin{axis}[
    width=0.24\textwidth, height=0.22\textwidth,
    title={$\sigma = 16$},
    grid=major, grid style={dotted, gray!40},
    xtick={0,2,4}, xticklabels={$1$, $10^2$, $10^4$},
    ytick={2,1,0,-1,-2}, yticklabels={$10^2$, $10^1$, $10^0$, $10^{-1}$, $10^{-2}$},
    tick label style={font=\tiny}, label style={font=\tiny},
]
\addplot[blue, thick, mark=*, mark options={fill=blue}, mark size=1pt]
    coordinates {(0,1.479)(1,1.021)(2,0.545)(3,0.061)(4,-0.436)(5,-0.957)};
\addplot[blue!30, draw=none, name path=normHi_m16, forget plot]
    coordinates {(0,1.698)(1,1.106)(2,0.615)(3,0.130)(4,-0.344)(5,-0.896)};
\addplot[blue!30, draw=none, name path=normLo_m16, forget plot]
    coordinates {(0,1.013)(1,0.914)(2,0.463)(3,-0.022)(4,-0.553)(5,-1.029)};
\addplot[blue!30, opacity=0.5, forget plot] fill between[of=normHi_m16 and normLo_m16];
\addplot[green!60!black, thick, mark=triangle*, mark options={fill=green!60!black}, mark size=1pt]
    coordinates {(0,-0.023)(1,-0.251)(2,-0.734)(3,-1.239)(4,-1.734)(5,-2.234)};
\addplot[green!30, draw=none, name path=mnHi_m16, forget plot]
    coordinates {(0,-0.014)(1,-0.216)(2,-0.705)(3,-1.196)(4,-1.682)(5,-2.196)};
\addplot[green!30, draw=none, name path=mnLo_m16, forget plot]
    coordinates {(0,-0.031)(1,-0.289)(2,-0.766)(3,-1.286)(4,-1.794)(5,-2.277)};
\addplot[green!60!black, opacity=0.2, forget plot] fill between[of=mnHi_m16 and mnLo_m16];
\addplot[red!70, thick, mark=square*, mark options={fill=red!70}, mark size=1pt, dashed]
    coordinates {(0,0.523)(1,-0.002)(2,-0.521)(3,-1.016)(4,-1.524)(5,-2.028)};
\addplot[red!30, draw=none, name path=softHi_m16, forget plot]
    coordinates {(0,0.691)(1,0.112)(2,-0.453)(3,-0.944)(4,-1.448)(5,-1.942)};
\addplot[red!30, draw=none, name path=softLo_m16, forget plot]
    coordinates {(0,0.246)(1,-0.159)(2,-0.601)(3,-1.103)(4,-1.615)(5,-2.136)};
\addplot[red!30, opacity=0.5, forget plot] fill between[of=softHi_m16 and softLo_m16];
\end{axis}
\end{tikzpicture}
\hfill
\begin{tikzpicture}
\begin{axis}[
    width=0.24\textwidth, height=0.22\textwidth,
    title={$\sigma = 64$},
    grid=major, grid style={dotted, gray!40},
    xtick={0,2,4}, xticklabels={$1$, $10^2$, $10^4$},
    ytick={2,1,0,-1,-2}, yticklabels={$10^2$, $10^1$, $10^0$, $10^{-1}$, $10^{-2}$},
    tick label style={font=\tiny}, label style={font=\tiny},
]
\addplot[blue, thick, mark=*, mark options={fill=blue}, mark size=1pt]
    coordinates {(0,2.084)(1,1.642)(2,1.125)(3,0.661)(4,0.109)(5,-0.352)};
\addplot[blue!30, draw=none, name path=normHi_m64, forget plot]
    coordinates {(0,2.260)(1,1.762)(2,1.221)(3,0.757)(4,0.199)(5,-0.266)};
\addplot[blue!30, draw=none, name path=normLo_m64, forget plot]
    coordinates {(0,1.783)(1,1.477)(2,1.001)(3,0.538)(4,-0.004)(5,-0.459)};
\addplot[blue!30, opacity=0.5, forget plot] fill between[of=normHi_m64 and normLo_m64];
\addplot[green!60!black, thick, mark=triangle*, mark options={fill=green!60!black}, mark size=1pt]
    coordinates {(0,-0.006)(1,-0.252)(2,-0.757)(3,-1.240)(4,-1.753)(5,-2.231)};
\addplot[green!30, draw=none, name path=mnHi_m64, forget plot]
    coordinates {(0,-0.005)(1,-0.197)(2,-0.705)(3,-1.193)(4,-1.710)(5,-2.181)};
\addplot[green!30, draw=none, name path=mnLo_m64, forget plot]
    coordinates {(0,-0.008)(1,-0.315)(2,-0.815)(3,-1.293)(4,-1.801)(5,-2.288)};
\addplot[green!60!black, opacity=0.2, forget plot] fill between[of=mnHi_m64 and mnLo_m64];
\addplot[red!70, thick, mark=square*, mark options={fill=red!70}, mark size=1pt, dashed]
    coordinates {(0,0.276)(1,-0.236)(2,-0.746)(3,-1.241)(4,-1.796)(5,-2.250)};
\addplot[red!30, draw=none, name path=softHi_m64, forget plot]
    coordinates {(0,0.431)(1,-0.114)(2,-0.641)(3,-1.142)(4,-1.733)(5,-2.172)};
\addplot[red!30, draw=none, name path=softLo_m64, forget plot]
    coordinates {(0,0.034)(1,-0.405)(2,-0.884)(3,-1.369)(4,-1.870)(5,-2.346)};
\addplot[red!30, opacity=0.5, forget plot] fill between[of=softHi_m64 and softLo_m64];
\end{axis}
\end{tikzpicture}
\hfill
\begin{tikzpicture}
\begin{axis}[
    width=0.24\textwidth, height=0.22\textwidth,
    title={$\sigma = 256$},
    grid=major, grid style={dotted, gray!40},
    xtick={0,2,4}, xticklabels={$1$, $10^2$, $10^4$},
    ytick={3,2,1,0,-1,-2}, yticklabels={$10^3$, $10^2$, $10^1$, $10^0$, $10^{-1}$, $10^{-2}$},
    tick label style={font=\tiny}, label style={font=\tiny},
]
\addplot[blue, thick, mark=*, mark options={fill=blue}, mark size=1pt]
    coordinates {(0,2.707)(1,2.159)(2,1.650)(3,1.177)(4,0.639)(5,0.129)};
\addplot[blue!30, draw=none, name path=normHi_m256, forget plot]
    coordinates {(0,2.905)(1,2.274)(2,1.742)(3,1.306)(4,0.750)(5,0.227)};
\addplot[blue!30, draw=none, name path=normLo_m256, forget plot]
    coordinates {(0,2.334)(1,2.002)(2,1.532)(3,0.994)(4,0.490)(5,0.002)};
\addplot[blue!30, opacity=0.5, forget plot] fill between[of=normHi_m256 and normLo_m256];
\addplot[green!60!black, thick, mark=triangle*, mark options={fill=green!60!black}, mark size=1pt]
    coordinates {(0,-0.002)(1,-0.292)(2,-0.786)(3,-1.271)(4,-1.778)(5,-2.281)};
\addplot[green!30, draw=none, name path=mnHi_m256, forget plot]
    coordinates {(0,-0.001)(1,-0.236)(2,-0.737)(3,-1.219)(4,-1.718)(5,-2.238)};
\addplot[green!30, draw=none, name path=mnLo_m256, forget plot]
    coordinates {(0,-0.002)(1,-0.356)(2,-0.842)(3,-1.329)(4,-1.848)(5,-2.328)};
\addplot[green!60!black, opacity=0.2, forget plot] fill between[of=mnHi_m256 and mnLo_m256];
\addplot[red!70, thick, mark=square*, mark options={fill=red!70}, mark size=1pt, dashed]
    coordinates {(0,0.076)(1,-0.429)(2,-0.927)(3,-1.423)(4,-1.939)(5,-2.447)};
\addplot[red!30, draw=none, name path=softHi_m256, forget plot]
    coordinates {(0,0.194)(1,-0.311)(2,-0.837)(3,-1.343)(4,-1.849)(5,-2.374)};
\addplot[red!30, draw=none, name path=softLo_m256, forget plot]
    coordinates {(0,-0.086)(1,-0.592)(2,-1.042)(3,-1.522)(4,-2.054)(5,-2.535)};
\addplot[red!30, opacity=0.5, forget plot] fill between[of=softHi_m256 and softLo_m256];
\end{axis}
\end{tikzpicture}
 
\par\smallskip
\centering
\begin{tikzpicture}
    \draw[blue, thick] (0,0) -- (0.6,0);
    \fill[blue] (0.3,0) circle (1.5pt);
    \node[right, font=\small] at (0.6,0) {$\|L^{\star} - \widetilde{M}\|$};

    \draw[green!60!black, thick] (5.5,0) -- (6.1,0);
    \node[green!60!black] at (5.8,0) {$\blacktriangle$};
    \node[right, font=\small] at (6.1,0) {$\|L^{\star} - \widetilde{M}\|_{\max}$};

    \draw[red!70, thick, dashed] (11.5,0) -- (12.1,0);
    \fill[red!70] (11.8,0) +(-2pt,-2pt) rectangle +(2pt,2pt);
    \node[right, font=\small] at (12.1,0) {$\|T^{\star} - \widetilde{T}^{\star}\|_{\max}$};
\end{tikzpicture}

\caption{$\|L^{\star} - \widetilde{M}\|$, $\|L^{\star} - \widetilde{M}\|_{\max}$ and $\norm{T^\star - \widetilde T^\star}_{\max}$ as functions of the sample size $s$,
for four attention heads with various values of $\sigma$, for the models Qwen3-14B and Mistral-7B.
Both axes show values on a log scale.
Shaded regions are mean$\pm$std across inputs.}
\label{fig:theoretical-validation-heads}
\end{figure}

\subsection{Model Training Experiments}\label{sec:model-training}
Next, we evaluate the effects of our approximation method on downstream model performance. 
To this end, we implemented a prototype of our positional LSH method for ALiBi in FlashAttention-2 \cite{daoflashattention}. 
Our code is available online\footnote{\url{https://github.com/DanielWolfsonTAU/positional_lsh}}. 
We note that our implementation is not optimized for performance from an engineering perspective (see also the discussion of our limitations in \Cref{sec:conclusion}). The goal of the experiments in this section is to gain empirical insight into the approximation quality of our method rather than demonstrate efficiency gains. 

We train publicly available models on publicaly available data:
\begin{itemize}
    \item Qwen3-0.6B \cite{qwen3technicalreport}, pre-trained from scratch on Wikitext-103 \cite{merity2016pointer} (103M raw words).
    \item Qwen3-4B \cite{qwen3technicalreport}, pre-trained from scratch on PG-19 dataset \cite{rae2019compressive} (1.97B raw words).
    \item Mistral-7B \cite{jiang2023mistral}, fine-tuned from its publicly available pre-trained checkpoint, on Wikitext-103.
\end{itemize}

We evaluate perplexity (PPL) and next-token prediction accuracy (Acc), both on the original context length used in training and under context length extrapolation (where the model is evaluated on longer context lengths than it was trained on). 
Qwen3-0.6B and Qwen3-4B are trained with an 8k context length, and Mistral-7B is trained with a 4k context length. We evaluate them at their trained context lengths as well as extrapolated context lengths up to 16k or 64k.

We evaluate Positional LSH with different sample sizes $s$. We compare it with the original models without positional bias and with ALiBi. 
As a baseline for comparison, we also include a ``fixed blocks'' baseline with a single block-diagonal mask matrix with a fixed block size of 512.

The results are displayed in \Cref{tab:results-qwen,tab:results-pg19,tab:results-mistral-4k}. For all models, as the number of samples in Positional LSH increases, performance improves, approaching that of ALiBi.
Furthermore, Positional LSH outperforms the fixed blocks baseline. 
Finally, for Qwen3-0.6B, Positional LSH outperforms the original model, even with a small sample size. 
For Qwen3-4B, this occurs only with $s=10$ samples and at the highest tested extrapolated context length of 64k. 
On Mistral-7B, neither positional bias method outperforms the original model in length extrapolation.


\begingroup
\setlength{\tabcolsep}{5.5pt} 
\begin{table}
  \caption{Length extrapolation results on Wikitext-103 for Qwen3-0.6B (trained from scratch)}
  \label{tab:results-qwen}
  \centering
  \footnotesize
  \begin{tabular}{lcccc}
    \toprule
    Context length: & \multicolumn{2}{c}{8k (training)} & \multicolumn{2}{c}{16k (extrapolation)} \\
    \cmidrule(lr){2-3} \cmidrule(lr){4-5}
    Method & PPL & Acc & PPL & Acc \\
    \midrule
    Original model                       & $20.213${\scriptsize$\pm 0.088$} & $44.796${\scriptsize$\pm 0.182$} & $21.316${\scriptsize$\pm 0.058$} & $43.652${\scriptsize$\pm 0.126$} \\
    ALiBi                                & $18.920${\scriptsize$\pm 0.009$} & $45.249${\scriptsize$\pm 0.035$} & $18.753${\scriptsize$\pm 0.021$} & $45.351${\scriptsize$\pm 0.020$} \\
    \midrule
    Positional LSH, $s=1$            & $19.217${\scriptsize$\pm 0.059$} & $45.051${\scriptsize$\pm 0.044$} & $19.278${\scriptsize$\pm 0.027$} & $44.985${\scriptsize$\pm 0.029$} \\
    Positional LSH, $s=10$           & $19.073${\scriptsize$\pm 0.054$} & $45.125${\scriptsize$\pm 0.025$} & $19.043${\scriptsize$\pm 0.070$} & $45.179${\scriptsize$\pm 0.021$} \\
    Positional LSH, $s=20$           & $19.050${\scriptsize$\pm 0.033$} & $45.146${\scriptsize$\pm 0.032$} & $18.961${\scriptsize$\pm 0.033$} & $45.192${\scriptsize$\pm 0.046$} \\
    Positional LSH, $s=50$           & $19.028${\scriptsize$\pm 0.048$} & $45.161${\scriptsize$\pm 0.018$} & $18.920${\scriptsize$\pm 0.020$} & $45.206${\scriptsize$\pm 0.041$} \\
    Positional LSH, $s=100$          & $19.017${\scriptsize$\pm 0.018$} & $45.172${\scriptsize$\pm 0.034$} & $18.893${\scriptsize$\pm 0.035$} & $45.219${\scriptsize$\pm 0.023$} \\
    \midrule
    Fixed blocks                         & $19.788${\scriptsize$\pm 0.054$} & $44.725${\scriptsize$\pm 0.058$} & $19.784${\scriptsize$\pm 0.041$} & $44.658${\scriptsize$\pm 0.067$} \\
    \bottomrule
  \end{tabular}
\vskip 1em

  \caption{Length extrapolation results on PG-19 for Qwen3-4B (trained from scratch)}
  \label{tab:results-pg19}
  \centering
  \footnotesize
  \begin{tabular}{lcccccccc}
    \toprule
    Context length: & \multicolumn{2}{c}{8k (training)} & \multicolumn{2}{c}{16k (ext.)} & \multicolumn{2}{c}{32k (ext.)} & \multicolumn{2}{c}{64k (ext.)} \\
    \cmidrule(lr){2-3} \cmidrule(lr){4-5} \cmidrule(lr){6-7} \cmidrule(lr){8-9}
    Method & PPL & Acc & PPL & Acc & PPL & Acc & PPL & Acc \\
    \midrule
    Original model          & $11.775$ & $47.481$ & $11.446$ & $47.838$ & $11.776$ & $47.353$ & $12.633$ & $46.183$ \\
    ALiBi                   & $11.751$ & $47.568$ & $11.445$ & $47.930$ & $11.450$ & $47.919$ & $11.402$ & $47.973$ \\
    \midrule
    Positional LSH, $s=1$   & $12.408$ & $46.777$ & $13.230$ & $45.936$ & $13.158$ & $45.927$ & $13.029$ & $46.086$ \\
    Positional LSH, $s=10$  & $12.298$ & $46.915$ & $12.315$ & $46.919$ & $12.318$ & $46.885$ & $12.218$ & $46.965$ \\
    \midrule
    Fixed blocks            & $12.841$ & $45.473$ & $13.511$ & $45.349$ & $13.515$ & $45.340$ & $13.462$ & $45.593$ \\
    \bottomrule
  \end{tabular}
\vskip 1em

  \caption{Length extrapolation results on Wikitext-103 for Mistral-7B (fine-tuned)}
  \label{tab:results-mistral-4k}
  \centering
  \footnotesize
  \begin{tabular}{lcccccc}
    \toprule
    Context length: & \multicolumn{2}{c}{4k (training)} & \multicolumn{2}{c}{8k (extrapolation)} & \multicolumn{2}{c}{16k (extrapolation)} \\
    \cmidrule(lr){2-3} \cmidrule(lr){4-5} \cmidrule(lr){6-7}
    Method & PPL & Acc & PPL & Acc & PPL & Acc \\
    \midrule
    Original model                       & $6.347${\scriptsize$\pm 0.001$} & $58.371${\scriptsize$\pm 0.016$} & $6.273${\scriptsize$\pm 0.002$} & $58.523${\scriptsize$\pm 0.016$} & $6.270${\scriptsize$\pm 0.004$} & $58.501${\scriptsize$\pm 0.039$} \\
    ALiBi                                & $6.498${\scriptsize$\pm 0.008$} & $57.924${\scriptsize$\pm 0.142$} & $6.455${\scriptsize$\pm 0.022$} & $58.086${\scriptsize$\pm 0.134$} & $6.461${\scriptsize$\pm 0.021$} & $57.960${\scriptsize$\pm 0.207$} \\
    \midrule
    Positional LSH, $s=1$            & $6.957${\scriptsize$\pm 0.009$} & $56.807${\scriptsize$\pm 0.031$} & $6.941${\scriptsize$\pm 0.015$} & $56.926${\scriptsize$\pm 0.025$} & $6.923${\scriptsize$\pm 0.011$} & $56.975${\scriptsize$\pm 0.027$} \\
    Positional LSH, $s=5$            & $6.725${\scriptsize$\pm 0.013$} & $57.186${\scriptsize$\pm 0.051$} & $6.711${\scriptsize$\pm 0.019$} & $57.325${\scriptsize$\pm 0.038$} & $6.721${\scriptsize$\pm 0.012$} & $57.235${\scriptsize$\pm 0.058$} \\
    Positional LSH, $s=10$           & $6.698${\scriptsize$\pm 0.010$} & $57.280${\scriptsize$\pm 0.026$} & $6.683${\scriptsize$\pm 0.011$} & $57.389${\scriptsize$\pm 0.028$} & $6.705${\scriptsize$\pm 0.034$} & $57.284${\scriptsize$\pm 0.021$} \\
    Positional LSH, $s=15$           & $6.672${\scriptsize$\pm 0.011$} & $57.324${\scriptsize$\pm 0.023$} & $6.664${\scriptsize$\pm 0.017$} & $57.412${\scriptsize$\pm 0.023$} & $6.668${\scriptsize$\pm 0.016$} & $57.338${\scriptsize$\pm 0.032$} \\
    \midrule
    Fixed blocks                         & $7.352${\scriptsize$\pm 0.014$} & $54.919${\scriptsize$\pm 0.022$} & $7.327${\scriptsize$\pm 0.017$} & $54.944${\scriptsize$\pm 0.019$} & $7.378${\scriptsize$\pm 0.015$} & $54.930${\scriptsize$\pm 0.013$} \\
    \bottomrule
  \end{tabular}
\end{table}
\endgroup

\subsection{Experimental Details} 
Experiments on all models except pre-training Qwen3-4B  were run on an NVIDIA H100 Tensor Core GPU (with 80GB VRAM). Experiments on Qwen3-4B were run on 8 NVIDIA B200 GPUs using Fully Sharded Data Parallel (FSDP) \cite{zhao2023pytorch}. 
The AdamW optimizer \cite{loshchilov2017decoupled} was used for training, with $\beta_1 = 0.9$, $\beta_2 = 0.999$, and $\epsilon = 10^{-8}$.
The hyperparameters were selected via evaluations performed on three different values for each parameter, spanning conventional value ranges.
For Qwen3-0.6B, a learning rate of $5 \times 10^{-4}$ is used, while for Mistral-7B the learning rate is $1 \times 10^{-4}$. The remaining
hyperparameters are shared across both models: 5 epochs, a weight decay of
0.01, and a warmup ratio of 0.1. 
A batch size of 1 is used for training and evaluation in both models, as larger batch sizes would exceed GPU memory limits and would compromise context length. 
The results of \Cref{sec:theoretical-validation} are averaged over 3 independent runs. 
Mistral-7B was fine-tuned using LoRA \cite{hu2021lora} with the following configuration: rank 16, $\alpha = 32$, dropout 0.05, no bias, targeting the query, key, value, and output projection matrices.

\section{Related Work}\label{sec:related}

Positional encodings and relative-position-aware attention in transformers are widely studied. 
Apart from RoPE, a central line of work injects positional information directly into attention through relative-position terms. Shaw et al.~\cite{shaw2018self} introduced relative position representations inside self-attention, and T5 \cite{raffel2020exploring} used a simplified scalar relative bias added to attention logits. Fixed and kernelized bias schemes for length extrapolation include ALiBi and KERPLE~\cite{chi2022kerple}, while FIRE \cite{lifunctional} provides a unified functional view. Jelassi et al.~\cite{jelassi2024repeat} also studied hard-ALiBi, a binary sliding-window variant with head-dependent window sizes, while Kazemnejad et al.~\cite{kazemnejad2023impact} studied NoPE, i.e., using no positional encoding at all. 
Our work is complementary to these methods: rather than proposing a new positional bias family, we propose a structural and algorithmic analysis of ALiBi.

From an efficiency perspective, exact attention with bias has recently been optimized at the systems level, including ALiBi-FlashAttention \cite{pli2024alibi} for exact computation and FlashBias \cite{wuflashbias} via a low-rank decomposition. Finally, LSH has a long history in similarity search \cite{indyk1998approximate,gionis1999similarity,charikar2002similarity,andoni2015practical} and kernel methods \cite{charikar2017hashing,siminelakis2019rehashing,backurs2019space,charikar2020kernel,coleman2020sub,gorodissky2026locality}, and has more recently been used for efficient attention and KV-cache compression \cite{kitaev2020reformer,han2024hyperattention,chenmagicpig,indykimproved,desaihashattention,joshirace}. Prior LSH-based attention work hashes token representations or semantic content; in contrast, we use LSH in a different way by hashing positions.

\section{Conclusions and Limitations}\label{sec:conclusion}

We introduced a positional LSH framework for attention with bias, establishing a formal connection between positional bias, binary attention masks, and positional embeddings. For ALiBi, this yields both a structural approximation of the bias matrix by randomized contiguous block masks and a uniform near-linear time approximation theorem for ALiBi-biased attention. Our experiments support this approximation perspective: as the number of samples increases, the approximation improves and downstream behavior approaches exact ALiBi.
Our framework is more general than ALiBi and may extend other fixed positional bias kernels.

Our contribution is theoretical and algorithmic rather than a systems result, and our experiments are of limited scale meant only to validate the theory. Although our analysis gives a near-linear asymptotic algorithm, we do not demonstrate wall-clock speedups over hardware-optimized exact ALiBi implementations. Some of this gap may be due to our implementation being under-optimized, but it likely also reflects the fact that current hardware and low-level kernels are highly optimized for large dense operations, so asymptotic gains from decomposing attention into many smaller local computations need not translate into speedups at the context lengths we experimented with. Identifying regimes where the favorable asymptotics of positional LSH lead to practical gains may involve much longer context lengths and is left an intriguing direction for future work.

\section*{AI Use Statement}
GPT 5.2 and 5.4 were used for writing assistance and polishing, literature search, and improving the clarity of mathematical presentation. GPT 5.4 also proposed the second (H\"older-based) proof of Lemma 3.7 after reviewing the authors' original (Fourier-based) proof. GPT and Claude were also used for interactive coding. The authors are accountable for the accuracy of all content and artifacts associated with this paper. 

\section*{Acknowledgments}
    This research was supported by the Israeli Ministry of Innovation, Science \& Technology, by Len
Blavatnik and the Blavatnik Family foundation, and by an Alon Scholarship.
This research was made possible through a GPU compute resource grant funded by the Association of University Heads, the Council for Higher Education and the AI Research Compute Center.

\bibliographystyle{amsalpha}
\bibliography{positionallsh}


\end{document}